\documentclass[preprint,3p]{elsarticle}
\usepackage{amsmath, amssymb, amsthm}

\usepackage{color}
\usepackage{lineno}
\usepackage{fourier}  

\usepackage{graphicx}
\usepackage{subcaption}
\usepackage{tikz}
\usepackage{bm}

\usepackage[noend]{algpseudocode}
\usepackage{algorithmicx, algorithm}

\newcommand{\bx}{\ensuremath{\boldsymbol{x}}}
\newcommand{\by}{\ensuremath{\boldsymbol{y}}}

\newcommand{\comment}[1]{}

\newcommand{\veps}{\varepsilon}
\newcommand{\bbN}{\mathbb{N}}
\newcommand{\bbR}{\mathbb{R}}
\newcommand{\bbZ}{\mathbb{Z}}
\newcommand{\diag}{\operatorname{diag}}
\renewcommand{\binom}[2]{C^{#2}_{#1}}
\newcommand{\vect}{\operatorname{vec}}
\newcommand{\lint}[1]{\lfloor #1 \rfloor}
\newcommand{\uint}[1]{\lceil #1 \rceil}

\definecolor{grey}{rgb}{0.75, 0.75, 0.75}
\definecolor{orange}{rgb}{1.0, 0.5, 0.5}
\definecolor{brown}{rgb}{0.5, 0.25, 0.0}
\definecolor{pink}{rgb}{1.0, 0.5, 0.5}

 \newtheorem{lemma}{Lemma}
 \newtheorem{proposition}{Proposition}
 \newtheorem{theorem}{Theorem}
 \newtheorem{definition}{Definition}
 \newtheorem{remark}{Remark}
\newtheorem{corollary}{Corollary}[lemma]
\numberwithin{equation}{section}

\DeclareMathOperator{\nnpar}{\triangledown}
\DeclareMathOperator{\nntpone}{{\otimes_1}}

\date{\today}

\begin{document}

\begin{frontmatter}
\title
{PowerNet: Efficient Representations of Polynomials and Smooth
  Functions by Deep Neural Networks with Rectified Power Units}

%

\author[addr2,addr1]{Bo Li\fnref{firstauthor}}
\author[addr2,addr1]{Shanshan Tang\fnref{firstauthor}}
\fntext[firstauthor]{Contributed equally. Author list is alphabetical.}
\author[addr1,addr2,addr3]{Haijun Yu\corref{cor1}}
\cortext[cor1]{Corresponding author.}
\address[addr1]{LSEC, Institute of Computational Mathematics and Scientific/Engineering Computing,\\
	Academy of Mathematics and Systems Science, Beijing 100190, China}
\address[addr2]{School of Mathematical Sciences, University of Chinese Academy of Sciences, Beijing 100049, China}
\address[addr3]{National Center for Mathematics and Interdisciplinary Sciences, Chinese Academy of Sciences, Beijing 100190, China}


\begin{abstract}
  Deep neural network with rectified linear units (ReLU) is getting more
  and more popular recently. However, the derivatives of the function
  represented by a ReLU network are not continuous, which limit the
  usage of ReLU network to situations only when smoothness is not
  required.  In this paper, we construct deep neural networks with
  rectified power units (RePU), which can give better approximations for
  smooth functions. Optimal algorithms are proposed to explicitly build
  neural networks with sparsely connected RePUs, which we call
  PowerNets, to represent polynomials with no approximation error. For
  general smooth functions, we first project the function to their
  polynomial approximations, then use the proposed algorithms to
  construct corresponding PowerNets. Thus, the error of best polynomial
  approximation provides an upper bound of the best RePU network
  approximation error. For smooth functions in higher dimensional
  Sobolev spaces, we use fast spectral transforms for tensor-product
  grid and sparse grid discretization to get polynomial
  approximations. Our constructive algorithms show clearly a close
  connection between spectral methods and deep neural networks: a
  PowerNet with $n$ layers can exactly represent polynomials up to
  degree $s^n$, where $s$ is the power of RePUs. The proposed
  PowerNets have potential applications in the situations where
  high-accuracy is desired or smoothness is required.
\end{abstract}



\end{frontmatter}


\section{Introduction}
Artificial neural network (ANN) has been a hot research topic for
several decades. Deep neural network (DNN), a special class of ANN with
multiple hidden layers, is getting more and more popular recently.
Since 2006, when efficient training methods were introduced by Hinton et
al \cite{hinton_fast_2006}, DNNs have brought significant improvements
in several challenging problems including image classification, speech
recognition, computational chemistry and numerical solutions of
high-dimensional partial differential equations, see e.g.
\cite{hinton_deep_2012, lecun_deep_2015, krizhevsky_imagenet_2017,
  han_solving_2017, zhang_deep_2018}, and references therein.

The success of ANNs rely on the fact that they have good representation power. 
Actually, the universal approximation property of neural networks is well-known:
neural networks with one hidden layer of continuous/monotonic sigmoid
activation functions are dense in continuous function space
$C([0, 1]^d)$ and $L^1([0, 1]^d)$, see
e.g. \cite{cybenko_approximation_1989, funahashi_approximate_1989,
  hornik_multilayer_1989} for different proofs in different settings.
Actually, for neural network with non-polynomial $C^\infty$ activation
functions, the upper bound of approximation error is of spectral type
even using only one-hidden layer, i.e.  error rate $\veps=n^{-k/d}$ can
be obtained theoretically for approximation functions in Sobolev space
$W^k([-1, 1]^d)$, where $d$ is the number of dimensions, $n$ is the
number of hidden nodes in the neural
network\cite{mhaskar_neural_1996}. 
It is believed that one
of the basic reasons behind the success of DNNs is the fact that deep
neural networks have broader scopes of representation than shallow
ones. Recently, several works have demonstrated or proved this in
different settings.  For example, by using the composition function
argument, Poggio et al \cite{poggio_why_2017} showed that deep networks
can avoid the curse of dimensionality for an important class of problems
corresponding to compositional functions. In the general function
approximation aspect, it has been proved by Yarotsky
\cite{yarotsky_error_2017} that DNNs using rectified linear units
(abbr. ReLU, a non-smooth activation function defined as
$\sigma_{1}(x):=\max\{0, x\}$) need at most
$\mathcal{O}(\veps^{\frac{d}{k}}(\log|\veps|+1))$ units and nonzero
weights to approximation functions in Sobolev space
$W^{k, \infty}([-1, 1]^d)$ within $\veps$ error.  This is similar to the
results of shallow networks with one hidden layer of $C^\infty$
activation units, but only optimal up to a $\mathcal{O}(\log|\veps|)$
factor.  Similar results for approximating functions in
$W^{k, p}([-1, 1]^d)$ with $p<\infty$ using ReLU DNNs are given by
Petersen and Voigtlaender\cite{petersen_optimal_2018}.  The
significance of the works by Yarotsky \cite{yarotsky_error_2017} and
Peterson and Voigtlaender \cite{petersen_optimal_2018} is that by using
a very simple rectified nonlinearity, DNNs can obtain high order
approximation property. Shallow networks do not hold such a good
property. Other works show ReLU DNNs have high-order approximation
property include the work by E and Wang\cite{e_exponential_2018} and
the recent work by Opschoor et al.\cite{opschoor_deep_2019}, the latter one relates
ReLU DNNs to high-order finite element methods.

A basic fact used in the error estimate given in
\cite{yarotsky_error_2017} and \cite{petersen_optimal_2018} is that
$x^2, xy$ can be approximated by a ReLU network with
$\mathcal{O}(\log|\veps|)$ layers. To remove this approximation error
and the extra factor $\mathcal{O}(\log|\veps|)$ in the size of neural
networks, we proposed to use rectified power units (RePU) to construct
exact neural network representations of polynomials
\cite{li_better_2019}. The RePU function is defined as
\begin{equation}
  \label{eq:RePU}
  \sigma_s(x) = \begin{cases} 
	x^s, & x \ge 0, \\
	0,   & x<0, 
  \end{cases}
\end{equation}
where $s$ is a non-negative integer. When $s=1$, we have the Heaviside
step function; when $s=1$, we have the commonly used ReLU function
$\sigma_1$.  We call $\sigma_2$, $\sigma_3$ rectified quadratic unit
(ReQU) and rectified cubic unit (ReCU) for $s=2, 3$, respectively.  Note
that, some pioneering works have been done by Mhaskar and his coworkers (see e.g. \cite{mhaskar_approximation_1993},
\cite{chui_neural_1994a}) to give an theoretical upper bound of DNN function
approximations by converting splines into RePU DNNs. However, for very smooth functions, their
constructions of neural network are not optimal and meanwhile are not numerically
stable.  The error bound obtained is quasi-optimal due to an extra $\log(k)$
factor, where $k$ is related to the smoothness of the underlying
functions. The extra $\log(k)$ factor is removed in our earlier
work\cite{li_better_2019} by introducing some explicit optimal and
stable constructions of ReQU networks to exactly represent
polynomials. In this paper, we extend the results to deep networks using
general RePUs with $s\ge 2$.

Comparing with other two constructive approaches (The Qin Jiushao
algorithm and the first-composition-then-combination method used in
\cite{mhaskar_approximation_1993}, \cite{chui_neural_1994a}, etc), our
constructions of RePU neural networks to represent polynomials are
optimal in the numbers of network layers and hidden nodes. To
approximate general smooth functions, we first approximate the function
by its best polynomial approximation, then convert the polynomial
approximation into a RePU network with optimal size. The conclusion of
algebraic convergence for $W^{k, 2}$ functions and exponential
convergence for analytic functions then follows straightforward.  For
multi-dimensional problems, we use the concept of sparse grid to improve
the error estimate of neural networks and lessen the curse of
dimensionality.

The main advantage of the ReLU function is that ReLU DNNs are relatively
easier to train than DNNs using other analytic sigmoidal activation
units in traditional applications. The latter ones have well-known
severe gradient vanishing phenomenon.  However, ReLU networks have some
limitations. E.g., due to the fact that the derivatives of a ReLU
network function are not continuous, ReLU networks are hard to train
when the loss function contains derivatives of the network, thus
functions with higher-order smoothness are desired. Such an example is
the deep Ritz method solving partial differential equations
(PDEs) recently developed by E and Yu\cite{e_deep_2018}, where ReQU networks are
used.

The remain part of this paper is organized as follows. In Section 2 we
first show how to realize univariate polynomials and approximate smooth
functions using RePU networks. Then we construct RePU network
realization of multivariate polynomials and general multivariate smooth functions in 
Section 3, with extensions to high-dimensional functions in sparse space given in Subsection 3.3.  
A short summary is given in Section 4.

\section{Approximation of univariate smooth functions}

We first introduce notations. Denote by $\bbN$ the set of all positive
integer, $\bbN_{0} := \{0\}\cup \bbN$,
$\bbZ_n := \{ 0, 1, \ldots, n-1 \}$ for $n \in \bbN$.
\begin{definition}
  We define a neural network $\Phi$ with input of dimension
  $d \in \bbN$, number of layer $L \in \bbN$ as a matrix-vector sequence
  \begin{equation} \label{eq:NNdef}
    \Phi =\big( (A_1, b_1), \cdots, (A_L, b_L) \big), 
  \end{equation}
  where $A_k, k=1, \ldots, L$ are $N_{k} \times N_{k-1}$ matrices,
  $b_{k}\in \bbR^{N_{k}\times 1}$ are vectors called bias, $N_{0}=d$ and
  $N_1, \cdots, N_{L}\in \bbN$.
\end{definition}

\begin{definition}
  If $\Phi$ is a neural network defined by \eqref{eq:NNdef}, and
  $\rho:\bbR \to \bbR$ is an arbitrary activation function, then define
  the neural network function
  \begin{align} \label{eq:NNfun}
    R_\rho(\Phi): \bbR^{d} \rightarrow \bbR^{N_L}, 
    \qquad R_\rho(\Phi)(\bx) = \bx_{L},
  \end{align}
  where $\bx_L = R_\rho (\Phi)(\bx)$ is defined as
  \begin{equation}
  \label{eq:NNfdef}
  \begin{cases}
	\bx_{0} := \bx,                         &                      \\
	\bx_{k} := \rho( A_k \bx_{k-1} + b_k ), & k=1, 2, \ldots, L-1, \\
	\bx_{L} := A_L \bx_{L-1} + b_L.  &
  \end{cases}
\end{equation}
Here we denote vector variables $\bx_k \in \bbR^{N_k}$ by bold letters
and use the definition
\begin{equation*}
  \rho(\by) 
  := \left( \rho(y^1), \cdots, \rho(y^m) \right)^T, 
  \quad 
  \forall\ \by = (y^1, \cdots, y^m)^T \in \bbR^m.
\end{equation*}
\end{definition}
We use three quantities to measure the complexity of a neural network
$\Phi$: number of layers $L(\Phi)$, number of nodes(i.e. activation
units) $N(\Phi)$, and number of nonzero weights $M(\Phi)$, which are
$L$, $\sum_{k=1}^{L-1} N_k(\Phi)$ and $\sum_{k=1}^L M_k(\Phi)$,
respectively. For the neural network defined in \eqref{eq:NNdef},
$N_k(\Phi) := N_k$, $k=0,\ldots, L$ are the dimensions of $\bm{x}_k$,
and $M_k(\Phi):=\| A_k \|_0 + \| b_k \|_0$ (for $k=1,\ldots, L$) is the
number of nonzero weights in the $k$-th affine transformation. Note
that, in this paper, we define $L$ as the layers of affine
transformations defined in \eqref{eq:NNfdef}.  We also call $\bx_0$ the
input layer, $\bx_L$ the output layer, and $\bx_k$, $k=1, \ldots, L-1$
hidden layers.  So, there are $L-1$ hidden layers, which is the number
of layers of activation units.

\begin{definition}
  We define $\Pi_{d, N, L}^m$ as the collection of all neural networks
  of input dimension $d$, output dimension $m$ with at most $N$ neurons
  arranged in $L$ layers, i.e.
  \begin{align}
  	\label{eq:NNSet}
  	\Pi^m_{d, N, L} 
  	&:= \left\{\, 
  		\Phi =\big( (A_1, b_1), \cdots, (A_L, b_L) \big)\,
  		\Big|\, 
  		\substack{
    	A_k\in \bbR^{N_k \times N_{k-1}},\ b_k\in \bbR^{N_k\times 1},\ \text{for}\ k=1, \ldots, L; \\
    	N_0=d,\ N_L=m,\ \sum_{k=1}^{L-1}N_k=N.\hfill
  		}
    	\,\right\}
  \end{align}
  For given activation function $\rho$, we further define
  \begin{align}
  	\label{eq:fNNSet}
  	\Pi^m_{d, N, L, \rho} 
    &:= \left\{\, R_\rho(\Phi) \mid \Phi \in \Pi^m_{d, N, L} \,\right\}.
  \end{align}  
\end{definition}

To construct complex networks from simple ones, We first introduce several network compositions.
\begin{definition}
  \label{def:concat}
  Let $L_1, L_2\in \bbN$ and
  $\Phi^1 =\big( (A^1_1, b^1_1), \ldots, (A^1_{L_1}, b^1_{L_1}) \big)$,
  $\Phi^2 =\big( (A^2_1, b^2_1), \ldots, (A^2_{L_2}, b^2_{L_2}) \big)$
  be two neural networks such that the input layer of $\Phi^1$ has the
  same dimension as the output layer of $\Phi^2$. 
  We define the the concatenation of $\Phi^1$ and $\Phi^2$ as
  \begin{align}
    \label{eq:concat}
    \Phi^2 \circ \Phi^1 
    &:= \Big( 
      (A^1_1,b^1_1), \ldots, (A^1_{L_1-1}, b^1_{L_1-1}), 
      (A^2_1 A^1_{L_1}, A^2_1 b^1_{L_1} + b^2_1), 
      (A^2_2, b^2_2), \ldots, (A^2_{L_2}, b^2_{L_2})
      \Big).
  \end{align}
\end{definition}
By the definition, we have
\begin{align*}
  & R_{\sigma_s}(\Phi^2 \circ \Phi^1) 
	= R_{\sigma_s}(\Phi^2) \Big( R_{\sigma_s}(\Phi^1) \Big)
	=: R_{\sigma_s}(\Phi^2) \circ R_{\sigma_s}(\Phi^1), \\
  & L(\Phi^{2}\circ\Phi^1) 
  	= L(\Phi^1) \:+\: L(\Phi^{2})-1, 
    \qquad\quad 
    N(\Phi^{2}\circ\Phi^1) 
    = N(\Phi^1) + N(\Phi^2).
\end{align*}

\begin{definition} \label{def:parallel} Let
  $\Phi^1 =\big( (A^1_1, b^1_1), \ldots, (A^1_L, b^1_L) \big)$,
  $\Phi^2 =\big( (A^2_1, b^2_1), \ldots, (A^2_L, b^2_L) \big)$ be two
  neural networks both with $L\in \bbN$ layers. Suppose the input
  dimensions of the two networks are $d_1, d_2$ respectively.  We define
  the parallelization of $\Phi^1$ and $\Phi^2$ as
  \begin{align}
    \label{eq:parallel}
    \Phi^1 \nnpar \Phi^2
    &:= \big(  (\tilde{A}_1, \tilde{b}_1), \cdots, (\tilde{A}_L, \tilde{b}_L) \big), 
  \end{align}
  where
  \begin{align*}
	\tilde{A}_1 
	&= \begin{bmatrix}
      \bar{A}^1_1 \\
      \bar{A}^2_1
    \end{bmatrix}, 
    & 
    \tilde{b}_1 
    & = \begin{bmatrix}
      b^1_1 \\
      b^2_1
    \end{bmatrix}, 
    & \quad \text{and} \quad 
    \tilde{A}_i 
    & = \begin{bmatrix}
      A^1_{i} & \bm{0}  \\
      \bm{0}  & A^2_i \\
	\end{bmatrix}, \quad \tilde{b}_i = \begin{bmatrix}
      b^1_{i} \\
      b^2_{i} \\
    \end{bmatrix}, 
    \quad \text{for}\quad 1< i \le L.
  \end{align*}
  Here $\bar{A}^i_1, i=1,2$ are formed from ${A}^i_1, i=1,2$
  correspondingly, by padding zero columns in the end to one of them
  such that they have same number of columns.  Obviously,
  $\Phi^1 \nnpar \Phi^2$ is a neural network with
  $\max\{d_1, d_2\}$-dimensional input and $L$ layers. We have the
  relationship
  \begin{align*}
    R_{\sigma_2}\big(\Phi^1 \nnpar \Phi^2\big)
    &= (R_{\sigma_2}(\Phi^1), R_{\sigma_2}(\Phi^2)), 
    && \\
    N\big(\Phi^1 \nnpar \Phi^2\big) 
    &= N(\Phi^1) + N(\Phi^2),
    &
      M\big(\Phi^1 \nnpar \Phi^2\big) 
    &= M(\Phi^1) + M(\Phi^2).
  \end{align*}
  For $\Phi^1$, $\Phi^2$ defined as above but not necessarily have same
  dimensions of input, we define the tensor product of $\Phi^1$ and
  $\Phi^2$ as
  \begin{align}
  \label{eq:tensorp}
    \Phi^1 \otimes \Phi^2 
    := \big(  (\tilde{A}_1, \tilde{b}_1), \cdots, 
    (\tilde{A}_L, \tilde{b}_L) \big), 
  \end{align}
  where
  \begin{align*}
	\tilde{A}_i = \begin{bmatrix}
      A^1_{i} & \bm{0}  \\
      \bm{0}  & A^2_i \\
    \end{bmatrix}, \quad \tilde{b}_i = \begin{bmatrix}
      b^1_{i} \\
      b^2_{i} \\
    \end{bmatrix}, \quad \text{for}\quad 1 \le i \le L.
  \end{align*}
  Obviously, $\Phi^1 \otimes \Phi^2$ is a $L$-layer neural network with
  $N_0(\Phi^1)+N_0(\Phi^2)$ dimensional input and
  $N_L(\Phi^1)+N_L(\Phi^2)$ dimensional output. We have the relationship
  \begin{align*}
    R_{\sigma_2}\big(\Phi^1 \otimes \Phi^2\big) 
    &= (R_{\sigma_2}(\Phi^1), R_{\sigma_2}(\Phi^2)), 
    & & \\
    N_k\big(\Phi^1 \otimes \Phi^2\big) 
    &= N_k(\Phi^1) + N_k(\Phi^2), \ \forall\, k=0,\ldots, L,
    &
      M_k\big(\Phi^1 \otimes \Phi^2\big) 
      &= M_k(\Phi^1) + M_k(\Phi^2)\ \forall\, k=1\ldots, L.
  \end{align*}
\end{definition}

\subsection{Basic properties of RePU networks}

Our analyses rely upon the fact: $x, x^{2}, \ldots, x^{s}$ and $xy$ can
all be realized by a one-hidden-layer $\sigma_s$ neural network with a
few number of coefficients, which is presented in the following lemma.

\begin{lemma} \label{lem:baseeq} The monomials $x^{n}, 1\le n \le s$ can
  be exactly represented by neural networks with one hidden layer of a
  finite number of $\sigma_{s}(x)\;(2\le s\in \bbN)$ activation
  nodes. More precisely:
  \begin{itemize}
    
  \item [(i)] For $s=n$, the monomial $x^{n}$ can be realized exactly
    using a $\sigma_{s}$ network having one hidden layer with two nodes
    as following,
    \begin{align}
      \label{eq:mono_s}
      x^{s} 
      &= \gamma^{T}_{0}\sigma_{s}(\alpha_{0}x), 
      &\quad
        \gamma_{0}
      &=\begin{bmatrix}
        1\\ (-1)^s\\
      \end{bmatrix}, 
      \quad
      \alpha_{0}=\begin{bmatrix}
        1\\ -1\\
      \end{bmatrix}.
    \end{align}  
    Correspondingly, the neural network is defined as
    \begin{equation}
    \label{eq:mono1_net}
    \Phi^1_{mo} 
    = \big( (\alpha_0, \bm{0}), (\gamma_0^T, 0) \big).
    \end{equation} 
    A graph representation of $\Phi^1_{mo}$ is sketched in
    Fig. \ref{fig:nn_mo_1}.
    
  \item [(ii)] For $1 \le n \le s$, the monomial $x^{n}$ can be realized
    exactly using a $\sigma_{s}$ network having only one hidden layer
    with no more than $2s$ nodes as
    \begin{align}
      \label{eq:mono_t}
      x^{n} 
      &= \gamma^{T}_{1, n}\sigma_s(\alpha_1 x + \beta_1) + \lambda_{0, n}, \quad n = 1, \ldots, s-1, 
    \end{align}
    where
    \begin{align}
      \label{eq:a1b1g1}
      \alpha_{1}
      &=\begin{bmatrix}
        \alpha_0 \\ 
        \vdots  \\
        \alpha_0
      \end{bmatrix} \in \bbR^{2s\times 1},
      & \beta_{1}
      &= \begin{bmatrix}
        b_1 \alpha_0 \\
        \vdots \\
        b_s \alpha_0
      \end{bmatrix} \in \bbR^{2s\times 1}
      & \gamma_{1, n}
      &= \begin{bmatrix}
        \lambda_{1, n} \gamma_0 \\
        \vdots \\
        \lambda_{s, n}\gamma_{0}  \\
      \end{bmatrix} \in \bbR^{2s\times 1}, 
    \end{align}
    Here $b_1, \ldots, b_s$ are distinct points in $\bbR$. We
    suggest to use \eqref{eq:bs2}-\eqref{eq:bs6} for $s\le 6$ and
    \eqref{eq:chebynodes} for $s>6$.
    $\lambda_{0, n}, \lambda_{1, n}, \ldots, \lambda_{s, n}$ are
    calculated by \eqref{eq:Lambda}.  The neural network is defined as
    \begin{equation}
      \label{eq:mono2n_net}
      \Phi^2_{mo, n} 
      = \big( (\alpha_1, \beta_1), 
      (\gamma_{1,n}^T, \lambda_{0,n}) \big).
    \end{equation}
    A graph representation of $\Phi^2_{mo,n}$ is sketched in
    Fig. \ref{fig:nn_mo_2n}.  Note that, when $n=0$, we have a trivial
    realization: $\alpha_1 = \beta_1 = \gamma_{1, 0} =0$,
    $\gamma_{0, 0} = 1$.  When $n=s$, the implementation in (i) is more
    efficient. When $n=1$, we obtain the network realization of identity
    function $\Phi_\text{idx} := \Phi^2_{mo, 1}$.
  \end{itemize}
\end{lemma}
\begin{proof}
  (1) It is easy to check that $x^{s}$ has an exact $\sigma_{s}$
  realization given by
  \begin{align}
    \rho_{s}(x) 
    :=
    \sigma_{s}(x) +(-1)^{s}\sigma_{s}(-x) 
    =
    \gamma^{T}_{0}\sigma_{s}(\alpha_{0}x).
  \end{align}
  
  (2) For the case of $1\le n \le s$, we consider the following linear
  combination
  \begin{align}
    \label{eq:und_co}
    \lambda_{0} + \sum^s_{k=1}\lambda_k\rho_s (x+b_k)
    = \lambda_0 +
    \sum^s_{k=1}\lambda_k\left(\sum^s_{j=0}C^j_{\!s} b^{s-j}_k x^j \right)
    = \lambda_{0} +
    \sum^s_{j=0}C^j_{\!s}\left(\sum^s_{k=1}\lambda_k b^{s-j}_k \right) x^j, 
  \end{align}
  where $\lambda_0, \lambda_k, b_k, k=1, 2, \ldots, s$ are parameters to
  be determined. $C^i_s, i=0, 1, \ldots, s$ are binomial coefficients.
  Identify the above expression with a polynomial of degree does not
  exceed $s$, i.e. $\sum^s_{k=0} d_k x^k$, we obtain the following
  linear system
  \begin{align}
    D_{s+1}\bm{\lambda_s}
    &:=
      \begin{bmatrix}
        1           & 1           & \cdots & 1           & 0      \\
        \vdots      & \vdots      &        & \vdots      & \vdots \\
        b^{s-i}_{1}  & b^{s-i}_{2} & \cdots & b^{s-i}_{s} & 0      \\
        \vdots      & \vdots      &        & \vdots      & \vdots \\
        b^{s-1}_{1}  & b^{s-1}_{2} & \cdots & b^{s-1}_{s} & 0      \\
        b^s_{1}     & b^s_{2}   & \cdots & b^s_{s}   & 1      \\
    \end{bmatrix}
    \begin{bmatrix}
      \lambda_{1} \\
      \vdots      \\
      \lambda_{i} \\
      \vdots      \\
      \lambda_{s} \\
      \lambda_{0}
    \end{bmatrix}
    =
    \begin{bmatrix}
      d_{s}(C^s_{\!s})^{-1} \\
      \vdots                  \\
      d_{i}(C^{i}_{\!s})^{-1} \\
      \vdots                  \\
      d_{1}(C^1_{\!s})^{-1} \\
      d_{0}(C^{0}_{\!s})^{-1}
    \end{bmatrix}, 
  \end{align}
  where the top-left $s\times s$ sub-matrix of $D_{s+1}$ is a Vandermonde
  matrix $V_{s}$, which is invertible as long as
  $b_{k}, \;k=1, 2, \ldots, s$ are distant. The choices of $b_k$ are
  discussed later in Remark \ref{rem:Vmatrix}.  Denote
  $\bm{\lambda}_s=[\lambda_1, \ldots, \lambda_s, \lambda_0]^T$,
  $\bm{b}=[b^s_1, \ldots, b^s_{s} ]^{T}$,
  $\bm{d}=[d_s, \ldots, d_{0}]^{T}$. We have
  \begin{align*}
    D_{s+1} 
    &=
    \begin{bmatrix}
      V_{s}      & \bm{0} \\
      \bm{b}^{T} & 1      \\
    \end{bmatrix}, 
    &
    D^{-1}_{s+1} 
    &=
    \begin{bmatrix}
      V^{-1}_{s}            & \bm{0} \\
      -\bm{b}^{T}V^{-1}_{s} & 1      \\
    \end{bmatrix}, 
  \end{align*}
  then
  \begin{align}
    \label{eq:Lambda}
    \bm{\lambda}_s 
    &=
    \begin{bmatrix}
      V^{-1}_{s}            & \bm{0} \\
      -\bm{b}^{T}V^{-1}_{s} & 1      \\
    \end{bmatrix}
    \diag \Big( (C^s_{\!s})^{-1}, (C^{s-1}_{\!s})^{-1}, 
    \cdots, (C^{0}_{\!s})^{-1} \Big)
    \bm{d}.
  \end{align}
  To represent $x^n\;(1\le n \le s)$, we have
  $\bm{d}=\bm{e}^{s+1}_{s-n+1}$ in \eqref{eq:Lambda}, where
  $\bm{e}^{s+1}_{k} := [\delta_{1,k}, \ldots, \delta_{s+1,k}]^T$ and
  $\delta_{i,k}$ is the Kronecker delta function.
\end{proof}

\begin{figure}[ht!]
  \begin{center}
    \includegraphics[width=0.8\textwidth]{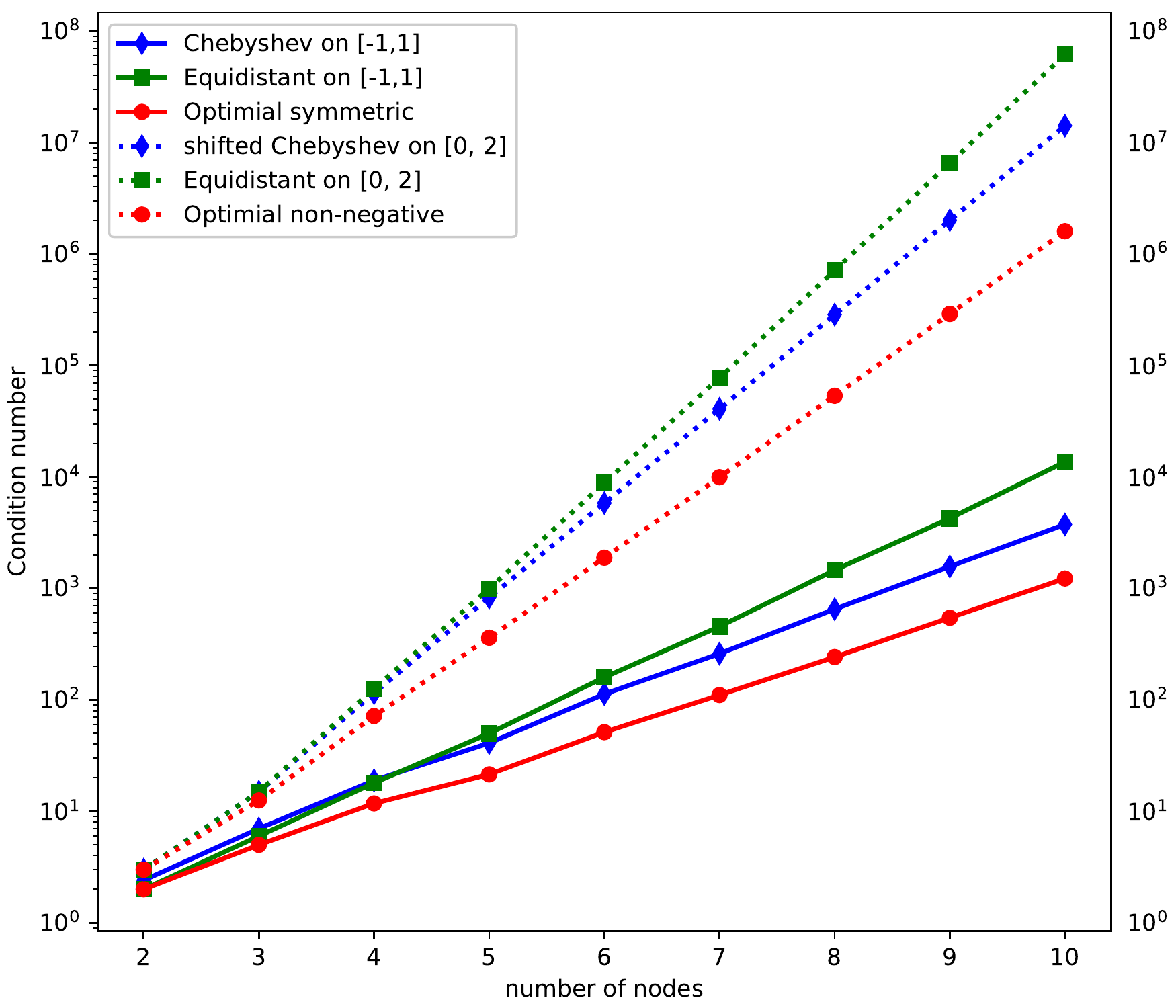}
    \caption{The growth of $l_\infty$ condition number of Vandermonde
      matrices $V_{s}$ corresponding to different sets of nodes
      $\{ b_k, \ k=1, \ldots, s\}$.  The data for optimal symmetric
      nodes and optimal non-negative nodes are from
      \cite{gautschi_optimally_2011}. }
    \label{fig:cond}
  \end{center}
\end{figure}

\begin{remark} \label{rem:Vmatrix} The inverse of Vandermonde matrix
  will inevitably be involved in the solution of (\ref{eq:Lambda}),
  which make the formula \eqref{eq:mono_t} difficult to use for large
  $s$ due to the geometrically growth of the condition number of the
  Vandermonde matrix \cite{gautschi_optimally_1975,
    beckermann_condition_2000, gautschi_optimally_2011}.  The condition
  number of the $s\times s$ Vandermonde matrices with three different
  choices of symmetric nodes are given in Figure \ref{fig:cond}. The
  three choices for symmetric nodes are Chebyshev nodes
  \begin{align}
    \label{eq:chebynodes}
    b_k 
    = 
    \cos\big( \tfrac{k-1}{s-1} \pi \big), \quad k=1, \ldots, s,
  \end{align}
  equidistant points
  \begin{align}
    b_k 
    = 1 - 2 \tfrac{k-1}{s-1}, \quad k=1, \ldots, s,
  \end{align}
  and numerically calculated optimal nodes.  The counterparts of these
  three different choices for non-negative nodes are also depicted in
  Figure \ref{fig:cond}. Most of the results are from
  \cite{gautschi_optimally_2011}.  For large $s$ the numerical procedure
  to calculate the optimal nodes may not succeed. But the growth rates
  of the $l_\infty$ condition number of Vandermonde matrices using
  Chebyshev nodes on $[-1,1]$ is close to the optimal case, so we use
  Chebyshev nodes \eqref{eq:chebynodes} for large $s$. For smaller
  values of $s$, we use numerically calculated optimal nodes, which are
  given for $2\le s \le 6$ in \cite{gautschi_optimally_1975}:
  \begin{align}
    & b_1 = -b_2 = 1, & s=2 \label{eq:bs2}\\
    & b_1 = -b_3 = \sqrt{3/2} \approx 1.2247448714,\ 
      b_2 = 0, & s=3 \label{eq:bs3}\\
    & b_1 = -b_4 \approx 1.2228992744,\
      b_2 = -b_3 \approx 0.5552395908, & s=4 \label{eq:bs4}\\
    & b_1 = -b_5 \approx 1.2001030479,\ 
      b_2 = -b_4 \approx 0.8077421768,\ 
      b_3 = 0, & s=5 \label{eq:bs5} \\
    &  b_1 = -b_6 = 1.1601101028, \
      b_2 = -b_5 = 0.9771502216, \
      b_3 = -b_4 = 0.3788765912, & s=6 \label{eq:bs6}
  \end{align}
  Note that, in some special cases, if non-negative nodes are used, the
  number of activation functions in the network construction can be
  reduced. However, due to the fact that the condition number in this
  case is larger than the case with symmetric nodes, we will not
  consider the use of all non-negative nodes in this paper.
\end{remark}

Based on Lemma \ref{lem:baseeq}, one can easily obtain following
results.

\begin{corollary} \label{corr:poly1layer} A univariate polynomial with
  degree up to $s$ can be exactly represented by neural networks with
  one hidden layer of $2s$ activation nodes.  More precisely, by
  (\ref{eq:mono_t}), we have
	\begin{align}
	\label{eq:poly1layer}
      \sum^{s}\limits_{j=0}d_{j}x^{j}
      = d_0 + 
      \sum^{s}\limits_{j=1}d_{j}\cdot
      \left(\gamma^{T}_{1, j}\sigma_{s}(\alpha_{1}x+\beta_{1}) + \lambda_{0, j}\right)
      =
      \tilde{\gamma}^{T}_{3}\sigma_{s}(\alpha_{1}x+\beta_{1}) +
      \tilde{c}_1, 
	\end{align}
	where $\tilde{\gamma}_{3}= \sum^{s}_{j=1}d_{j}\gamma_{1, j}$,
    $\tilde{c}_{1} = d_0+\sum^{s}_{j=1}d_{j}\lambda_{0, j}$.  The
    corresponding neural network is defined as
	\begin{equation}
	\label{eq:poly1_net}
	\Phi^1_{po} (\bm{d}) 
	= \big( (\alpha_1, \beta_1), (\tilde{\gamma}_{3}^T, \tilde{c}_{1})
    \big),
	\end{equation}
	where $\bm{d} = [d_s, \ldots, d_1, d_0]^T$. A graph representation
    of $\Phi^1_{po}$ is sketched in Fig. \ref{fig:nn_po_1}.
\end{corollary}

\begin{figure}
  \centering
  \begin{subfigure}[b]{0.45\textwidth}
    \centering \includegraphics[width=\linewidth]{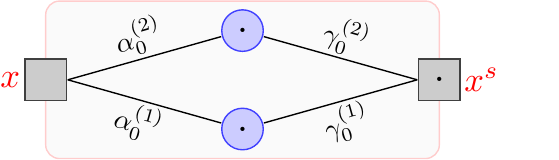} 
    \caption{The $\Phi^1_{mo}$ defined in \eqref{eq:mono1_net}.}
    \label{fig:nn_mo_1}
  \end{subfigure}\hfill
  \begin{subfigure}[b]{0.45\textwidth}
    \centering \includegraphics[width=\linewidth]{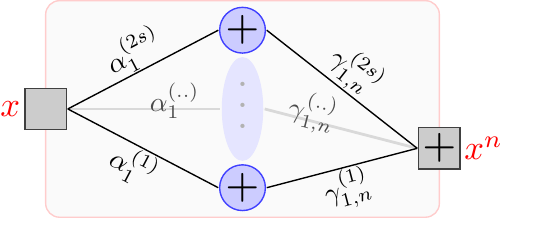}
    \caption{The $\Phi^2_{mo,n}$ defined in \eqref{eq:mono2n_net}}
    \label{fig:nn_mo_2n}
  \end{subfigure}
	\\ 
	\medskip

  \begin{subfigure}[b]{0.45\textwidth}
	\centering \includegraphics[width=\linewidth]{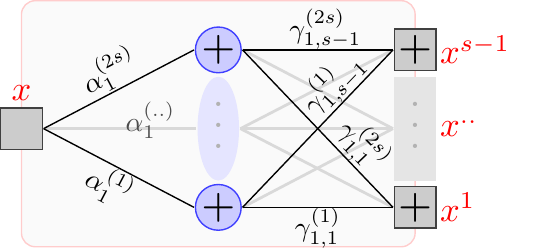}
	\caption{The $\Phi_{c}$ used in \eqref{eq:phi_b1}.}
	\label{fig:nn_mo_c}
  \end{subfigure}\hfill
  \begin{subfigure}[b]{0.45\textwidth}
    \centering
    \includegraphics[width=\linewidth]{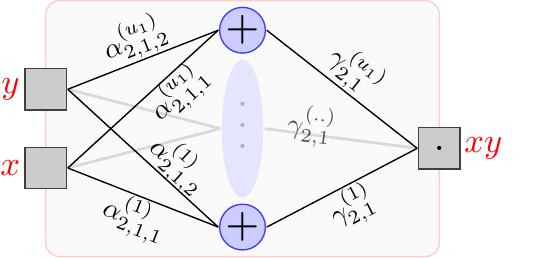}
    \caption{The $\Phi^1_{bm,1}$ defined in \eqref{eq:mono3n_net}}
    \label{fig:nn_bm_11}
  \end{subfigure}
	\\
	\medskip
	
  \begin{subfigure}[b]{0.45\textwidth}
	\centering
	\includegraphics[width=\linewidth]{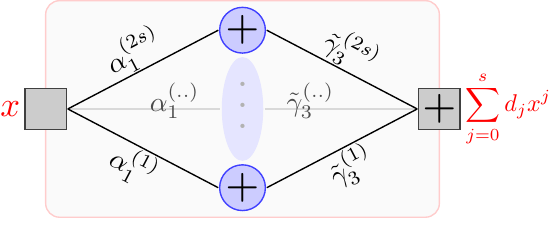}
	\caption{The $\Phi^1_{po}(\bm{d})$ defined in \eqref{eq:poly1_net}}
	\label{fig:nn_po_1}
  \end{subfigure}\hfill
  \begin{subfigure}[b]{0.45\textwidth}
	\centering
	\includegraphics[width=\linewidth]{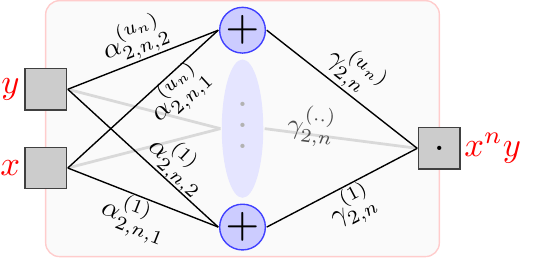}
	\caption{The $\Phi^1_{bm,n}$ defined in \eqref{eq:mono3n_net}.}
	\label{fig:nn_bm_1n}
  \end{subfigure}
	\\
	\medskip
	
  \begin{subfigure}[b]{0.45\textwidth}
	\centering
	\includegraphics[width=\linewidth]{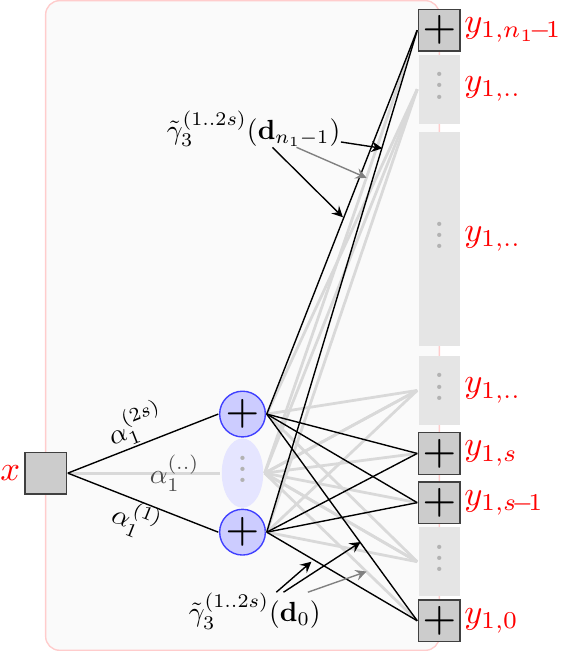}
	\caption{The $\Phi^1_{\bm{a}}$ to realize $\{\, y_{1,k} \,\}$ defined in \eqref{eq:y1}}
	\label{fig:nn_po_y1}
  \end{subfigure}\hfill
  \begin{subfigure}[b]{0.45\textwidth}
    \centering
    \includegraphics[width=\linewidth]{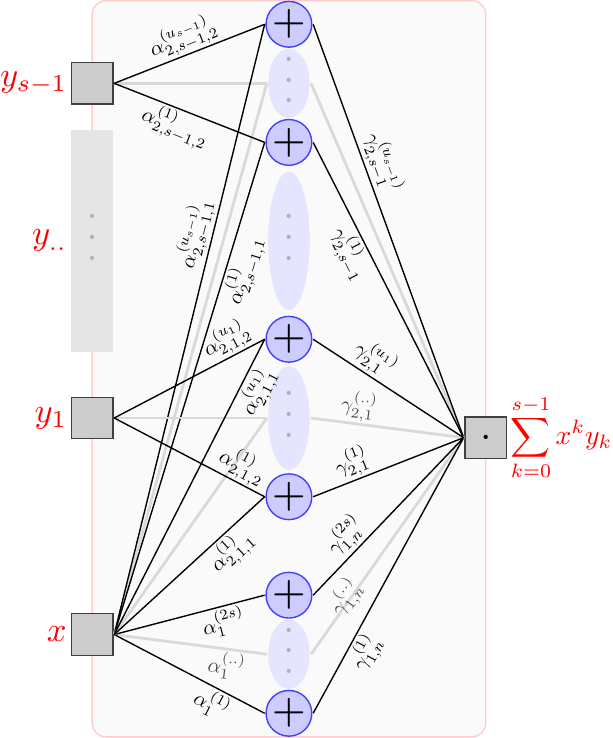}
    \caption{The $\Phi^1_{pm}$ defined in \eqref{eq:pxny_net}}
    \label{fig:nn_pm_1}
  \end{subfigure}
  \caption{Some shallow neural networks used as building bricks of the
    RePUs DNNs. Here circles represent hidden nodes, squares represent
    input, output and intermediate variables, A ``$+$'' sign inside a
    circle or a square represent a nonzero bias.}
  \label{fig:nn}
\end{figure}

In the implementation of polynomials, operations of the form $x^n y$ will be
frequently involved.  Following lemma asserts that
$x^n y, 0\le n \le s-1$ can be realized by using only one hidden layer.
\begin{lemma} \label{lem:xny} Bivariate monomials $x^n y$,
  $0 \le n \le s-1$ can be realized as a linear combination of at most
  $u_n$ activation units of $\sigma_s(\cdot)$ as
	\begin{align}
	\label{eq:bivamono}
      x^{n}y
      &= \gamma^{T}_{2, n}\sigma_{s}(\alpha_{2, n, 1}x + \alpha_{2, n, 2}y + \beta_{2, n}), 
        \quad n=0, 1, \ldots, s-1, 
	\end{align}
	where
    $\alpha_{2, n, 1}, \alpha_{2, n, 2}, \beta_{2, n}, \gamma_{2, n} \in
    \bbR^{u_n\times 1}$, $u_n = 2(n+1)(s-n)$.  A particular formula is
    given by \eqref{eq:xny_sigs_coef} in the appendix section. The
    corresponding neural network is defined as
	\begin{equation}
      \label{eq:mono3n_net}
      \Phi^1_{bm,n}
      = \big( ([\alpha_{2,n,1}, \alpha_{2,n,2}], \beta_{2,n}), (\gamma_{2,n}^T, 0) \big).
	\end{equation}
	A graph representation of $\Phi^1_{bm,n}$ is sketched in
    Fig. \ref{fig:nn_bm_1n}.  Obviously, the numbers of nonzero weights
    in the first layer and second layer affine transformation are $3u_n$
    and $u_n$ correspondingly.
\end{lemma}
The proof of Lemma \ref{lem:xny} is lengthy. We put it in the appendix
section.

\begin{corollary} \label{corr:xnyp} A polynomial of the form
  $\sum_{k=0}^{s-1} x^k y_k$ can be realized as a linear combination of
  at most $w$ activation units of $\sigma_s(\cdot)$ as
	\begin{equation}
      \label{eq:pxny_net}
      \Phi^1_{pm}
      = \Phi^0 \circ \big( 
      \Phi^2_{mo,1} \nnpar \Phi^1_{bm,1} \nnpar \cdots \nnpar \Phi^1_{bm, s-1} \big).
	\end{equation}
	Here $\Phi^0=\big((\bm{1}_s,0)\big)$ with
    $\bm{1}_s = (1,\ldots, 1)^T \in \bbR^{s\time 1}$ contains only a
    linear combination layer.  A graph representation of $\Phi^1_{pm}$
    is sketched in Fig. \ref{fig:nn_pm_1}.
	The numbers of nonzero weights in the first layer and second layer
    affine transformations are at most $3w$ and $w$ correspondingly. Here
   	\begin{align*}
    w 
    &= 2s + \sum_{k=1}^{s-1} u_k  
    = 2s + \sum_{k=1}^{s-1} 2(k+1)(s-k) 
    = \frac13(s^3+3s^2+2s)
    \end{align*}
    
\end{corollary}		

\subsection{Optimal realizations of polynomials by RePU networks with no
  error}

The basic properties of $\sigma_s$ given in Lemma \ref{lem:baseeq} and
Lemma \ref{lem:xny} can be used to construct neural network
representation of any monomial and polynomial.  We first present the
results of monomial.

For $x^n$ with $1\le n \le s$, by Lemma \ref{lem:baseeq}, the number of
layers, hidden units and nonzero weights required in a $\sigma_s$
network to realize it is no more than $2$, $2s$, $6s+1$,
correspondingly. For $n>s$, we have the following Theorem.

\begin{theorem}
  \label{thm:xn1d}
  For $2\le s<n\in \bbN$, there exist a $\sigma_s$ network
  $\Phi^{3}_{mo}$ with
  \begin{align*}
    & L(\Phi^3_{mo}) \le \uint{\log_s n} + 1, 
      \qquad \quad
      N(\Phi^3_{mo}) \le \lint{\log_{s}n} \big( (s+1)^2/2 + 2 \big) + 2s \\
    & M(\Phi^3_{mo}) \le \big( \lint{\log_{s}n} - 1\big) (u^2+3u+4) + 2su + 4u + 4s +2, 
      \quad u:=(s+1)^2/4
  \end{align*}  
  to exactly represent the monomial $x^n$ defined on $\bbR$.  Here,
  $\lfloor x \rfloor$ represents the largest integer not exceeding $x$,
  and $\lceil x \rceil$ represents the smallest integer no less than
  $x$, for $x\in\bbR$.
\end{theorem}
\begin{proof} 
  1) For $n>s$, $\log_s n \notin \bbZ$, we first express $n\in \bbN$ in
  positional numeral system with radix $s$ as follows:
  \begin{equation}
    n 
    = n_m\cdot s^m + n_{m-1}\cdot s^{m-1} + \cdots + n_1\cdot s + n_0
    =: 
    \overline{(n_{m}\cdots n_{1}n_{0})_{s}}, 
  \end{equation}
  where $m=\lfloor\log_s n\rfloor$, $n_j\in \bbZ_s$ for
  $j=0,\ldots, m-1$ and $0\neq n_m \in \bbZ_s$. Then
  \begin{align}
    x^n
    = x^{n_m s^m}\cdot x^{\sum\limits_{j=0}^{m-1}n_js^{j}}.
 \end{align}
 Introducing intermediate variables
  \begin{equation}\label{eq:xidef1}
    \xi_k^{(1)}
    := x^{s^k}, 
    \qquad
    \xi_k^{(2)}
    := 
    x^{\sum\limits_{j=0}^{k-1}n_j s^j}, 
    \qquad \text{for}\ 1\le k\le m+1,
  \end{equation}
  then $x^n=\xi_{m+1}^{(2)}$ can be calculated iteratively as
  \begin{align}\label{eq:xirec}
    \left\{
    \begin{array}{lll}
      \xi_1^{(1)}=x^s, 
      &  \xi_1^{(2)}=x^{n_0}, 
      &  k=1, \\
      \xi_{k}^{(1)} = (\xi_{k-1}^{(1)})^s, 
      & \xi_{k}^{(2)} = (\xi_{k-1}^{(1)})^{n_{k-1}} \xi_{k-1}^{(2)}, 
      &  2\le k\le m, \\
      \xi_{m+1}^{(2)} = (\xi_m^{(1)})^{n_m} \xi_m^{(2)}, 
      & 
      &  k=m+1.
    \end{array}   
        \right.
  \end{align}
  Therefore, to construct a $\sigma_s$ neural network expressing $x^n$,
  we need to realize three basic operations: $(\cdot)^s$,
  $(\cdot)^{n_j}$ and multiplication. By Lemma \ref{lem:baseeq}, each
  step of iteration \eqref{eq:xirec} can be realized by a $\sigma_s$
  network with one hidden layer. Then the overall neural network to
  realize $x^n$ is a concatenation of those one-layer sub-networks.  We
  give the construction process of the neural network as follows.
  \begin{itemize}
  \item For $k=1$, the first sub-network
    $\Phi^1\!\!=\!\big( (A^1_1, b^1_{1}), (A^1_{2}, b^1_{2})
    \big)=\Phi^1_{mo}\nnpar\Phi^2_{mo,n_0}$ is constructed
    according to Lemma \ref{lem:baseeq} as
    \begin{align}
      \label{eq:m1}
      \begin{array}{l}
        \bm{x}^{(1)}_{0} 
        = x, \\
        \bm{x}^{(1)}_{1} 
        = \sigma_{s}\left( 
				        \begin{bmatrix}
                          \alpha_{0} \\
                          \alpha_{1}
				        \end{bmatrix}
        \bm{x}^{(1)}_{0}  +
        \begin{bmatrix}
          0 \\
          \beta_{1}
        \end{bmatrix}
        \right)
        =: 
        \sigma_{s}\left( A^1_1 \bm{x}^{(1)}_{0} + b^1_1 \right), \\
        \bm{x}^{(1)}_2 
        = \big[ \xi^{(1)}_1, \xi^{(2)}_{1} \big]^{T}
        = \begin{bmatrix}
          \gamma^{T}_{0} & \hspace{-0.2cm} 0  \\
          0 & \gamma^{T}_{1, n_{0}}
        \end{bmatrix}
              \bm{x}^{(1)}_1 + \begin{bmatrix}
                0 \\
                \lambda_{0, n_{0}}
              \end{bmatrix}
	    =: A^1_2 \bm{x}^{(1)}_1 + b^1_2.
      \end{array}
    \end{align}
    It is easy to see that the number of nodes in the hidden layer is
    $2(s+1)$, and the number of non-zeros in $A^1_{1}$ and $b^1_{1}$ is
    $4s+2$.

  \item For $k=2, \ldots, m$ the sub-network
    $\Phi^{k}\!\!=\!\big( (A^{k}_1, b^{k}_{1}), (A^{k}_{2}, b^{k}_{2})
    \big) = \Phi^1_{mo} \nnpar \Phi_{bm, n_{k-1}}$ are
    constructed as
    \begin{align}
      \label{eq:m2}
      \begin{array}{l}
        \bm{x}^{(k)}_0 
        = \big[\xi^{(1)}_{k-1}, \ \xi^{(2)}_{k-1} \big]^{T}, \\
        \bm{x}^{(k)}_1 
        = \sigma_{s}\left(
        \begin{bmatrix}
          \alpha_{0}           & 0 \\
          \alpha_{2, n_{k-1}, 1} & \alpha_{2, n_{k-1}, 2}
        \end{bmatrix}
                                   \bm{x}_{0}
                                   + \begin{bmatrix}
                                     0 \\
                                     \beta_{2, n_{k-1}}
                                   \end{bmatrix}
        \right)
        =: \sigma_{s}\left( A^k_1 \bm{x}^{(k)}_0 + b^k_1 \right), \\
        \bm{x}^{(k)}_2  
        = \big[ \xi^{(1)}_{k}, \xi^{(2)}_{k} \big]^{T}
        = \begin{bmatrix}
          \gamma^{T}_{0} & 0                      \\
          0              & \gamma^{T}_{2, n_{k-1}}
        \end{bmatrix}
                           \bm{x}^{(k)}_{1}
                           =: A^{k}_{2}\bm{x}^{(k)}_{1} + b^{k}_{2}.
      \end{array}
    \end{align}
    The number of nodes in layer $k $ is $2(n_{k-1}+1)(s-n_{k-1})+2$,
    and the number of non-zeros in $A^{k}_{1}$ and $b^{k}_{1}$ is at
    most $6(n_{k-1}+1)(s-n_{k-1})+ 2 \le 3(s+1)^2/2 + 2$. The number of
    non-zeros in $A^{k}_{2}$ and $b^{k}_{2}$ is at most
    $2(n_{k-1}+1)(s-n_{k-1})+ 2 \le (s+1)^2/2 + 2$.

  \item For $k=m+1$, the sub-network
    $\Phi^{m+1}\!\!=\!\big( (A^{m+1}_1, b^{m+1}_{1}), (A^{m+1}_{2},
    b^{m+1}_{2}) \big) = \Phi_{bm, n_{m}}$ is constructed as
    \begin{align}\label{eq:m3}
      \begin{array}{l}
        \bm{x}^{(m+1)}_{0} 
        = \big[\xi^{(1)}_{m}, \ \xi^{(2)}_{m} \big]^{T},\\
        \bm{x}^{(m+1)}_{1} 
        = \sigma_{s}\Big( [ \alpha_{2, n_{m}, 1}, \ \alpha_{2, n_{m}, 2}] 
        \bm{x}^{(m+1)}_{0}
        + \beta_{2, n_{m}}  \Big)
        =: \sigma_{s}(A^{m+1}_{1}\bm{x}^{(m+1)}_{0}+b^{m+1}_{1}), \\
        \bm{x}^{(m+1)}_{2} 
        = x^{n} 
 	    = \gamma^{T}_{2, n_{m}} \bm{x}^{(m+1)}_{1}
		=: A^{m+1}_{2}\bm{x}^{(m+1)}_{1} + b^{m+1}_{2}.
      \end{array}
    \end{align}
    By a straightforward calculation, we get the number of nodes in
    Layer $m+1$ is at most $2(n_m+1)(s-n_m)$, and the number of
    non-zeros in $A^{m+1}_{2}$ and $b^{m+1}_{2}$ is
    $2(n_{m}+1)(s-n_{m})\le (s+1)^2/2$.
  \end{itemize}
  \begin{figure}[htb!]
	\begin{center}
      \resizebox{0.9\textwidth}{!}{
        \begin{tikzpicture}
          \path (-5, -1.0) node (x0) {Input: $x$};
			
          \path (-4.2, 1.1) node (x110) {$(1)$}; \path (-4.2, 0.5) node
          (x11) {$x^{s^{1}}$}; \path (-4.2, -1.5) node (x21) {}; \path
          (-4.2, -2.5) node (x31) {$x^{n_{0} }$};
			
          \path (-2.5, 1.1) node (x120) {$(2)$}; \path (-2.5, 0.5) node
          (x12) {$x^{s^{2}}$};

			\path (-3.1, -2.5) node (x3200) {};
			\path (-3.1, -2.5) node (x3201) {};
			\path (-2.5, -2.3) node (x32) {$x^{\sum^{1}\limits_{j=0}n_{j}s^{j}}$};
			\path (-2.3, -2.5) node (x3210) {};
			
			\path (-0.5, 1.1) node (x110) {$(3)$};
			\path (-0.5, 0.5) node (x13) {$x^{s^{3}}$};
			
			\path (-0.8, -1.35) node (x230) {};
			
			\path (-1.1, -2.5) node (x3300) {};
			\path (-1.1, -2.5) node (x3301) {};
			\path (-0.5, -2.3) node (x33) {$x^{\sum^{2}\limits_{j=0}n_{j}s^{j}}$};
			\path (-0.2, -2.5) node (x3310) {};
			
			\path (1.1, 0.5) node (x14) {};
			\path (1.25, -1.3) node (x24) {};
			\path (1.1, -1.5) node (x240) {};
			\path (1.1, -2.5) node (x34) {};
			
			\filldraw [black] (1.4, 0.5) circle (1.3pt);
			\filldraw [black] (1.8, 0.5) circle (1.3pt);
			\filldraw [black] (2.2, 0.5) circle (1.3pt);
			
			\filldraw [black] (1.4, -1.0) circle (1.3pt);
			\filldraw [black] (1.8, -1.0) circle (1.3pt);
			\filldraw [black] (2.2, -1.0) circle (1.3pt);
			
			\filldraw [black] (1.4, -2.5) circle (1.3pt);
			\filldraw [black] (1.8, -2.5) circle (1.3pt);
			\filldraw [black] (2.2, -2.5) circle (1.3pt);
			
			\path (3.3, 1.1) node (x160) {$(m-2)$};
			\path (3.3, 0.5) node (x16) {$x^{s^{m-2}}$};
			\path (2.9, -1.3) node (x26) {};
			\path (3.3, -1.7) node (x261) {};

			\path (2.6, -2.5) node (x3600) {};
			\path (3.2, -2.5) node (x3601) {};
			\path (3.3, -2.3) node (x36) {$x^{\sum^{m-3}\limits_{j=0}s^{j}n_{j} }$};
			
			\path (5.3, 1.1) node (x170) {$(m-1)$};
			\path (5.3, 0.5) node (x17) {$x^{s^{m-1}}$};

			\path (4.6, -2.5) node (x3700) {};
			\path (5.2, -2.5) node (x3701) {};
			\path (5.3, -2.3) node (x37) {$x^{\sum^{m-2}\limits_{j=0}s^{j}n_{j} }$};
			
			\path (7.3, 1.1) node (x180) {$(m)$};
			\path (7.3, 0.5) node (x18) {$x^{s^{m}}$};
			\path (7.3, -1.5) node (x28) {};
			\path (6.6, -2.5) node (x3800) {};
			\path (7.6, -2.1) node (x3801) {};
			\path (7.3, -2.3) node (x38) {$x^{\sum^{m-1}\limits_{j=0}s^{j}n_{j}}$};
			
			\path (8.5, 0.5) node (x19) {};
			\path (8.1, -0.9) node (x2900) {};
			\path (8.1, -1.1) node (x2901) {};
			\path (8.7, -0.7) node (x29) {$x^{\sum^m\limits_{j=0}s^j n_j }$};
			\path (8.1, -2.5) node (x39) {};
			
			\draw[->, very thick] (x0) --(x11);
			\draw[->, very thick] (x0) --(x31);
			
			\draw[->, very thick] (x11)--(x12);
			\draw[->, very thick] (x11)--(x3200);
			
			\draw[->, very thick] (x12)--(x13);
			\draw[->, very thick] (x12)--(x3300);
			
			\draw[->, very thick] (x13)--(x14);
			\draw[->, very thick] (x13)--(x34);
			
			\draw[->, very thick] (x16)--(x17);
			\draw[->, very thick] (x16)--(x3700);
			
			\draw[->, very thick] (x17)--(x18);
			\draw[->, very thick] (x17)--(x3800);
			
			\draw[->, very thick] (x18)--(x2900);
			
			\draw[->, very thick] (x31)--(x3201);
			\draw[->, very thick] (x3210)--(x3301);
			\draw[->, very thick] (x3310)--(x34);
			
			\draw[->, very thick] (x3601)--(x3700);
			\draw[->, very thick] (x3701)--(x3800);
			\draw[->, very thick] (x3801)--(x2901);
			\end{tikzpicture}
		}
	\end{center}
	\caption{Sketch of a $\sigma_s$ network realization of $x^{n}$.
      Here $(k)$, $k=1,\ldots, m$ on the top part represents the
      intermediate variables of $k$-th hidden layer (the quantities
      beneath $(k)$).}
	\label{fig:xn1ds}
\end{figure}
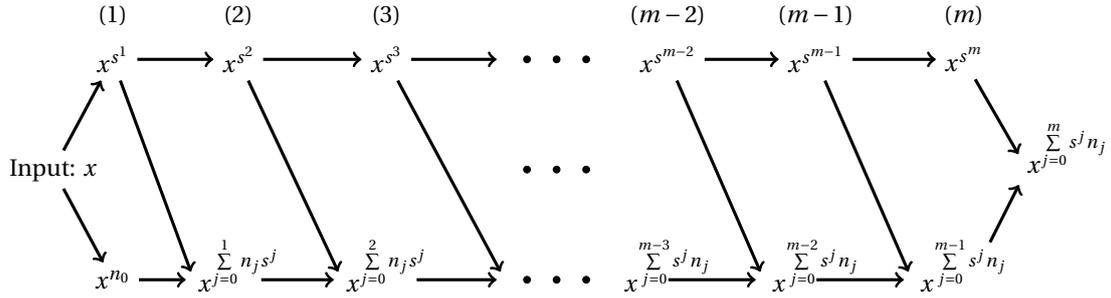
The whole neural network $\Phi^{3}_{mo}$ is constructed by a
concatenation of the all sub-networks, i.e.
    \begin{align}
      \Phi^{3}_{mo} 
      &= \Phi^{m+1}\circ\cdots\circ\Phi^1 \nonumber\\
      &= \Big(
        (A^1_1, b^1_{1}), 
        (A^2_{1}A^1_{2}, A^2_{1}b^1_{2}+b^2_{1}), \ldots, 
        (A^{m+1}_{1}A^{m}_{2}, A^{m+1}_{1}b^{m}_{2}+b^{m+1}_{1}), 
        (A^{m+1}_{2}, b^{m+1}_{2})
        \Big).
        \label{eq:mono_concat}
   \end{align}
   The network structure is sketched in Figure \ref{fig:xn1ds}. According
   to the definition, $L(\Phi^3_{mo})=m+2$.  The total number of nodes
   is given by
   \begin{align*}
     N(\Phi^3_{mo}) 
     = \sum_{k=1}^{m+1} N(\Phi^i) 
     &= 2(s+1) + 2(n_m+1)(s-n_m) + \sum_{k=2}^{m} (2(n_{k-1}+1)(s-n_{k-1})+2)  \\
     &\le m(s+1)^2/2 + 2m + 2s.
  \end{align*} 
  The number of non-zeros is given by
  \begin{align*}
    M(\Phi^3_{mo}) 
    &= ( \|A^1_1\|_0, \|b^1_1\|_0 ) 
      +  ( \|A^{m+1}_2\|_0, \|b^{m+1}_2\|_0)
      + \sum_{k=1}^{m}  \|A^{k+1}_1 A^k_2 \|_0 
      + \| A^{k+1}_1 b^k_2+b^{k+1}_1 \|_0 \\
    &= (4s+2)  + 2(n_m+1)(s-n_m) \\
    &\quad {}+ [u_0(2s+3)+4] +[ {u}_m(u_{m-1}+3) ] 
      +  \sum_{k=2}^{m-1} u_k (u_{k-1}+3) + 4\\
    &\le (m-1) (u^2 + 3u + 4) + 2s u + 4u + 4s + 2,
  \end{align*} 
  where $u_k = 2(n_k+1)(s-n_k) \le u := (s+1)^2/2$.
  
  2) For $n>s$, $\log_s n =m \in \bbZ$, we have $x^n = x^{s^m}$, which
  can be realized by a concatenation of $m$ shallow networks realizing
  $x^s$. So the number of layers, nodes and nonzero weights in this
  network realization is $m+1$, $2m$ and $4m$, correspondingly.
\end{proof}
\begin{algorithm}[hbt!]
  \caption{PNet\_Monomial$(s, n)$}
  \label{alg:monomial}
	\hspace*{0.02in} {\bf Input:}  $n\in \bbN,\; 2\le s \in\bbN$\\
	\hspace*{0.02in} {\bf Output:}  $\Phi_{mo}$ \qquad (with property $R_{\sigma_s}(\Phi_{mo})(x) = x^n$).
	\begin{algorithmic}[1]

		\If{$n \le 1$}
		\State Form $\Phi_{mo}=\big( (\delta_{n,1}, \delta_{n,0}) \big)$.

		\ElsIf{$n=s$}
		\State Form $\Phi_{mo}=\Phi_{mo}^1=\big( (\alpha_0, \bm{0}), (\gamma_0^T, 0) \big)$ with 
				$\gamma_0, \alpha_0$ defined in \eqref{eq:mono_s}.
		
		\ElsIf{$2 \le n < s$} %
		\State Form $\Phi_{mo}=\Phi_{mo,n}^2 = \big( (\alpha_1, \beta_1), (\gamma_{1, n}^T, \lambda_{0, n}) \big)$ with
				 $\alpha_1, \beta_1,\gamma_{1, n}^T, \lambda_{0, n}$  given in Lemma \ref{lem:baseeq}.(ii).
	
		\ElsIf{$n>s$}
		\State Let $m=\lfloor\log_{s}n\rfloor$  
		
		\If{$2^m = n$}
		\State Form $\Phi_{mo} = \big( (\alpha_0, \bm{0}), (\alpha_0\gamma_0^T, \bm{0}), \ldots, (\alpha_0\gamma_0^T, \bm{0}), (\gamma_0^T, 0) \big)$ with $m+1$ layers.
		
		\Else
		\State (1) Take $(n_{m}, \ldots, n_1, n_{0})$ such that $n = \overline{(n_m\cdots n_1 n_0)}_s$, 
		\State (2) Form $(A^1_{j}, b^1_{j})\ (j=1, 2)$ given in \eqref{eq:m1}.
		\State (3) Form $(A^{k}_{j}, b^{k}_{j})\ (j=1, 2)$ given in \eqref{eq:m2} for $k=2, \ldots, m$.
		\State (4) Form $(A^{m+1}_{j}, b^{m+1}_{j})\ (j=1, 2)$ given in \eqref{eq:m3}.
		\State (5) Form $\Phi_{mo} = \Phi^{3}_{mo}$ as defined in \eqref{eq:mono_concat}.

		\EndIf
		\EndIf
		\State 	\Return $\Phi_{mo}$.
	\end{algorithmic}
\end{algorithm}

\begin{remark}
	It is easy to check that: For any neural network with only one hidden $\sigma_s$ layer, the corresponding neural network function is a piecewise polynomial of degree $s$,
	for any neural network with $k$ hidden $\sigma_s$ layers, the corresponding network function is a piecewise polynomial of degree $s^k$. So $x^n$ can't be exactly represented by a $\sigma_s$ neural network with less than $\uint{\log_s n} -1$ hidden layers.
\end{remark}
\begin{remark}
  The detailed procedure presented in Lemma \ref{lem:baseeq} and Theorem
  \ref{thm:xn1d} is implemented in Algorithm \ref{alg:monomial}.  Note
  that this algorithm generates $\sigma_s$ DNN to represent monomial
  $x^n$ with least(optimal) hidden layers. For large $n$ and $s$, the numbers of
  nodes and nonzero weights in the network is of order
  $\mathcal{O}(s^2\log_s n )$ and $\mathcal{O}(s^4\log_s n )$,
  respectively, which are not optimal.  To lessen the size of the
  constructed network for large $s$, one may implement
  $(\xi^{(1)}_{k-1} )^{n_{k-1}} \xi^{(2)}_{k-1}$ in \eqref{eq:xirec} in
  two steps: i) implement $z = (\xi^{(1)}_{k-1} )^{n_{k-1}}$; ii)
  implement $z\,\xi^{(2)}_{k-1}$.  According to Lemma \ref{lem:baseeq}
  and Lemma \ref{lem:xny}, This will lessen both the number of nodes and
  the number of nonzero weights in the overall network but will add
  one-more hidden layer. To keep the paper tight, we will not present the detailed
  implementation of this approach here. Instead, we will describe this
  approach in the $\sigma_s$ network realization of polynomials.
\end{remark}

Now we consider converting univariate polynomials into $\sigma_s$
networks. If we directly apply Lemma \ref{lem:baseeq} and Theorem
\ref{thm:xn1d} to each monomial term in a polynomial of degree $n$ and
then combine them together, one would obtain a network of depth
$\mathcal{O}({\lceil\log_s n\rceil})$ and size at least
$\mathcal{O}({sn\lceil\log_s n\rceil})$, which is not optimal in terms
of network size. Fortunately, there are several other ways to realize
polynomials.  Next, we first discuss two straightforward constructions.
The first one is a direct implementation of Horner's method (also known
as Qin Jiushao's algorithm):
\begin{align}
  f(x) 
  &= a_0  + a_1 x + a_2 x^2 + a_3 x^3 + \ldots + a_n x^n \nonumber \\
  &= a_0 + x\bigg(a_1 + x\Big(a_2 + x\big( a_3+ \ldots +
    x(a_{n-1} + x a_n) \big) \Big)\bigg). 
    \label{eq:horner}
\end{align}
To describe the algorithm iteratively, we introduce the following
intermediate variables
\begin{equation*}
  y_k 
  = 
  \begin{cases}
	a_{n-1} + x a_n,     & k=n,                    \\
	a_{k-1} + x y_{k+1}, & k= n-1, n-2, \ldots, 1.
  \end{cases}
\end{equation*}
Then we have $y_0 = f(x)$. But an iterative implementation of $y_k$ using
realizations given in Lemma \ref{lem:baseeq}, \ref{lem:xny} and stack the
implementations up, we obtain a $\sigma_s$ neural network with $n$ layers
and each hidden layer has $4(s-1)$ activation units.

The second way is the method used by Mhaskar and his coworkers(see e.g.
\cite{mhaskar_approximation_1992, mhaskar_approximation_1993}), which is
based on the following proposition~\cite{chui_realization_1993,
  arora_understanding_2016}.

\begin{proposition} \label{prop:polynd} Let $m\geq 0, \ d\geq 1$ be
  integers. Then every polynomial in $d$ variables with total degree not
  exceeding $m$ can be written as a linear combination of
  $\binom{m+d}{d}$ quantities of the from
  $\left(\sum^{d}_{j=1}\omega_{j}x^{(j)}+b\right)^{m}$.
\end{proposition}

Suppose $p(\bx)$ is a polynomial of degree up to $n$ on $\bbR$, let
$\rho_s(x)=x^{s}$, $r = \uint{\log_s n}$ and define
\begin{align}
  g
  &:= \underbrace{\rho_s\circ \cdots \circ \rho_s}\limits_{r\ times}
    = x^{s^{r}}.
\end{align}
Then, according to Proposition \ref{prop:polynd}, one can find a network
work realization of $p(\bx)$ as
\begin{align} 
\label{eq:Pi}
  p(\bx) 
  &=
    \sum^{N}\limits_{k=1}c_{k}g(\omega_{k} \bx + b_{k})
    \in \Pi^1_{d, N, 2, g} 
    \in \Pi^1_{d, 2rN, r+1, \sigma_s}
\end{align}
where $N:=\binom{s^r+d}{d}$.  For $n \gg d$, we need to use
$\mathcal{O}(n^d\uint{\log_s n})$ nodes in
$\mathcal{O}(\uint{\log_s n})$ layers by using \eqref{eq:Pi}.

\begin{remark}
  The Horner's method and Mhaskar's Method have different properties.
  The first one is optimal in the number of nodes but use too many
  hidden layers; the latter one is optimal in the number of hidden
  layers, but the number of nodes is not optimal.  Another issue in the
  latter approach is that one has to calculate the coefficients
  $c_k, \omega_k, b_k$ in \eqref{eq:Pi}, which is not an easy task. Note
  that, when $d=1$, Proposition 2.1 is equivalent to Lemma
  \ref{lem:baseeq} and Corollary \ref{corr:poly1layer}, from which we
  see one need to solve some Vandermonde system to obtain the
  coefficients. The Vandermonde matrix is known has very large condition
  number for large dimension. A way to avoid solving a Vandermonde
  system is demonstrate in the proof of Lemma \ref{lem:xny}. However,
  from the explicit formula given in
  \eqref{eq:xny_rhos}-\eqref{eq:xny_gammas}, we see when $s$ is big,
  large coefficients with different signs coexist, which is deemed to have
  a large cancellation error.  So, lifting the activation function from
  $\rho_s$ to $\rho_n$ directly is not a numerically stable approach.
\end{remark}

Now, we propose a construction method that avoids the problem of solving
large Vandermonde systems. At the same time, the networks we constructed
have no very large coefficients.

Consider a polynomial $p(x)$ with degree $n$ greater than
$s$. Let $m=\lint{\log_s n}$.  We first use a recursive procedure
similar to the monomial case to construct a network with minimal layers.
\begin{enumerate}
\item[i)] Let $n_1=\uint{n/s}$, we consider $p(x)$ as a polynomial of degree
  $n_1 s$ by adding zero high degree monomials if $n_1 s>n$.  To realize
  $p(x) = \sum_{k=0}^{n_1 s} a_k x^k$ using a $\sigma_s$ network, we
  first divide the summation into $n_1$ groups as
  \begin{align*}
	p(x) 
	= \sum_{k=0}^{n_1s} a_k x^k 
	&= \sum_{k=0}^{n_1-2}  \sum_{j=0}^{s-1} a_{ks+j} x^{ks+j}
   + \sum_{j=0}^{s} a_{(n_1-1)s+j} x^{(n_1-1)s+j}
   \qquad \big(\,a_{k}=0\ \text{for}\ k>n\,\big)\\
	&= \sum_{k=0}^{n_1-2}\left( \big( x^s\big)^k \sum_{j=0}^{s-1} a_{ks+j} x^{j} \right)
   + \big( x^s\big)^{n_1-1} \sum_{j=0}^{s} a_{(n_1-1)s+j} x^{j}\\
	&= \sum_{k=0}^{n_1-1} z_1^k y_{1,k},
	\end{align*}
	where
	\begin{align}
      \label{eq:y1}
      z_1 
      &= x^s, & 
      y_{1,k} 
      &= \sum_{j=0}^{s-1} a_{ks+j} x^{j}
                              \quad\text{for}\ k=0,\ldots, n_1-2, &
      y_{1,n_1-1} 
      &= \sum_{j=0}^{s} a_{(n_1-1)s+j} x^{j}.
	\end{align}
	The above quantities
    $\left\{\, z_1,\ y_{1,k},\, k=0,\ldots, n_{1}-1\,\right\}$ can be
    realized by a $\sigma_s$ network
    $\Phi^1_a = \Phi^1_{mo} \nnpar \Phi^1_{\bm{a}} \in
    \Pi^{n_1+1}_{1, N_1, 2}$ with one hidden layer, where
    $\Phi^1_{\bm{a}}$ is implemented in Fig. \ref{fig:nn_po_y1}. The
    number of hidden nodes, and numbers of nonzero weights could be as
    small as
	\begin{align}
      N(\Phi^1_a) 
      &= 2+ 2s,
      &
        M_1(\Phi^1_a)
      &=2+4s, 
      &
        M_2(\Phi^1_a) 
      &= 2+\sum_{k=0}^{n_1-1} (2s+1),
	\end{align} 
	where $2s+1$ in the last term means each $y_{1,k}$ depending on $2s$
    nodes and $1$ shift value.  After above procedure, we have reduced
    the original univariate polynomial of degree $n$ to a polynomial of
    degree $n_1$. Note that here $\{ y_{1,k} \}$, $z_1$ are all variables.
	
  \item[ii)] Define $n_2 = \uint{n_1/s}$. For the resulting polynomial we can
    use similar procedure to get
	\begin{align*} 
      p(x) = \sum_{k=0}^{n_2 s - 1} z_1^k y_{1,k}  
      = \sum_{k=0}^{n_2-1} \sum_{j=0}^{s-1} z_1^{ks+j} y_{1,ks+j} 
      &= \sum_{k=0}^{n_2 - 1} (z_1^s)^k \sum_{j=0}^{s-1} z_1^{i} y_{1,ks+j}  \\
      &= \sum_{k=0}^{n_2 - 1} z_2^k y_{2,k},
	\end{align*}
	where
	\begin{align}
      z_2
      &=z_1^s, 
      & 
        y_{2,k}
      &=\sum_{j=0}^{s-1} z_1^{j} y_{1,ks+j} 
        \quad \text{for}\ k=0, \ldots, n_2-1,
	\end{align} 
	which, according to Lemma \ref{lem:baseeq} and Lemma \ref{lem:xny},
    can be realized by a neural network $\Phi^2_a$ of only one hidden
    layer. 
    More precisely, 
    \begin{align}
    \label{eq:nn_po_2a}
    \Phi^2_a 
    = \Phi^1_{mo} \nnpar 
    	\left(
		     \big(\nntpone\big)_{k=1}^{n_{2}} \Phi^1_{pm} \right),
    \end{align}
    where $\Phi^1_{pm}$ is defined in \eqref{eq:pxny_net}, a graph
    representation is sketched in Fig. \ref{fig:nn_pm_1}.  The operator
    $\nntpone$ is similar to $\otimes$ but all the sub-nets share one
    common input, which is taken as the first input of the composited net.
    
    The number of nodes, and numbers of nonzero weights are
	\begin{align}
      N(\Phi^2_a)
      &= 2 + n_2 \sum_{j=0}^{s-1} u_j = 2+n_2 w  
      &\\
      M_1(\Phi^2_a) 
      &= 2 + 3 n_2 \sum_{j=0}^{s-1} u_j = 2+3n_2w 
      & 
        M_2(\Phi^2_a) 
      &= 2 +  n_2 \sum_{j=0}^{s-1} u_j = 2+n_2w
	\end{align}
	
  \item[iii)] For $i=2,\ldots, m$, repeat similar procedure as ii).  Let
    $n_{i+1} = \uint{n_{i}/s}$, and using a $\sigma_s$ network
    $\Phi^{i+1}_a$ with one hidden layer to realize
	\begin{align}
      z_{i+1} 
      &= z_{i}^s, 
      & 
        y_{i+1,k} 
      &= \sum_{j=0}^{s-1} z_{i}^j y_{i, ks+j}, 
        \quad \text{for}\ k=0, \ldots, n_{i+1}-1.
	\end{align} 
	The number of nodes and nonzero weights are similar to the second
    step with $n_2$ replaced by $n_{i+1}$.  The recursive procedure ends at
    $i=m$. We obtain this conclusion by looking at the base-$s$ form of
    $n$: $\overline{(k_m\cdots k_1k_0)}_s$.  Noticing
    $\lint{n/s}=\overline{(k_m\cdots k_1)}_s$, which has $m$
    digits, and $n_1 = \uint{n/s}$ could be larger than
    $\lint{n/s}$ by one, which means $n_1$ either has $m$ digits or
    equal to $s^m$. The case that $\uint{n_i/s}$ has one more digit
    than $\lint{n_i/s}$ could happen only one in the recursive
    procedure. So $n_m$ has either one digit or equal to $s$, in both
    case, we have $p(x) = y_{m+1,0}$.
	
  \item[iv)] We obtain a $\sigma_s$ network realization of $p(x)$ by taking a
    concatenation of all the sub-networks in each iteration.
	\begin{equation}
      \Phi^2_{po} 
      = 
      \Phi^{m+1}_a \circ \Phi^m_a \circ \cdots \circ \Phi^1_a.
	\end{equation}
	Its number of nodes and nonzero weights are
	\begin{align}
      N(\Phi^2_{po}) 
      &= \sum_{k=1}^{m+1} N(\Phi^i_a) = 2+2s + \sum_{i=2}^{m+1} (2+n_iw) 
        = \mathcal{O}\Big(\frac{s^2+3s+2}3 \frac{s}{s-1}n \Big), \\
      M(\Phi^2_{po}) 
      &= 2+4s + \sum_{i=1}^{m} 
        \left( 2N(\Phi^{i+1}_a) +  w \sum_{k=0}^{n_{i+1}} w \right) 
		+ N(\Phi^{m+1}_a) 
        = \mathcal{O}\Big(\frac{(s^2+3s+2)^2}9 \frac{s}{s-1}n \Big).
	\end{align}
	
\end{enumerate}
The above construction produces a network with $m+2$ layers which is optimal. 
But the numbers of nodes and nonzero weights are not optimal for large values of $s$.
Next, we present an alternative construction method in following theorem that is
optimal in both number of layers and number of nodes.
\begin{theorem} \label{thm:Pn1d} If $p(x)$ is a polynomial of degree $n$
  on $\bbR$, then it can be represented exactly by a $\sigma_s$ neural
  network with {$\uint{\log_{s}n}+2$} layers, and number of nodes and
  non-zero weights are of order $\mathcal{O}(n)$ and $\mathcal{O}(sn)$,
  respectively.
\end{theorem}

\begin{proof}
  1) For polynomials of degree up to $s$, the formula
  (\ref{eq:poly1layer}) in Corollary \ref{corr:poly1layer} presents a
  one-hidden-layer network realization that satisfies the theorem.
   
  2) Below, we give a realization with much less number of nodes and
  nonzero weights by adding one-more hidden layer.  We describe the new
  construction in following steps.
\begin{itemize}
	
\item[i)] The first sub-network calculate $z_0 = x^s$ and $z_{0,1} = x$
  using
  \begin{equation}
    \label{eq:phi_b0}
    \Phi^0_{b} 
    = 
    \Phi^{1}_{mo} \nnpar \Phi_\text{idx} 
    \in \Pi_{1, N_1, 2}^2,
  \end{equation}
  where the number of nodes in this sub-network is $N_1 = 2+2s$.
	
\item[ii)] In the second sub-network, we calculate
  \begin{align*}
    z_1 
    &= z_0^s;
    & 
      z_{1,j} 
    &= z_0^j,\ j=1,\ldots, s-1; 
    & 
      y_{1,k} 
    &= \!\!\!\sum_{j=0}^{s-1+\delta_{n_1-1,k}} a_{ks+j} z_{0,1}^j,\ k=0,\ldots, n_1-1,
  \end{align*}
  which can be implemented as
  \begin{equation}
	\label{eq:phi_b1}
	\left\{ 
      \begin{aligned}
	\Phi^1_{b} 
	&= \Phi^1_{mo} \nnpar \big( \Phi_c \otimes \Phi^1_{\bm{a}} \big), 
	&
	\Phi_c 
	&= \big( (\alpha_1, \beta_1), (\gamma_c, \lambda_c)  \big), 
	&
	\Phi^1_{\bm{a}} 
	&= \big( (\alpha_1, \beta_1), (A_{\bm{a}}, b_{\bm{a}}) \big),
	\\
	\gamma_c 
	&= (\gamma_{1,1},\ldots, \gamma_{1,s-1})^T, 
	&
	A_{\bm{a}} 
	&= (\tilde{\gamma}_3(\bm{a}_0), \ldots,\tilde{\gamma}_3(\bm{a}_{n_1-1}) )^T, 
	&\\
	\lambda_c 
	&= (\lambda_{0,1},\ldots, \lambda_{0,s-1})^T, 
	&
	b_{\bm{a}} 
	& = (\tilde{c}_1(\bm{a}_0), \ldots, \tilde{c}_1(\bm{a}_{n_1-1}) )^T,
	&&
	\end{aligned}
  \right.
  \end{equation}
  where
  $\bm{a}_k = (\delta_{n_1-1,k}\cdot a_{ks+s}, a_{ks+s-1}, \ldots,
  a_{ks})$.  $\Phi_{c}$ is a network to realize
  $\{ z_{1,j} \mid j=1,\ldots,s-1 \}$, which is sketched in Fig. \ref{fig:nn_bm_11}.  
  $\Phi^1_{\bm{a}}$ is a network to
  realize $\{ y_{1,k} \mid k=0,\ldots,n_1-1 \}$, which is sketched in Fig. \ref{fig:nn_po_y1}.  
  Note that, according
  to Lemma \ref{lem:baseeq} and Corollary \ref{corr:poly1layer}, the
  number of nodes in $\Phi^1_b$ is $N_2 = 2+4s$.

\item[iii)] For $i=2,\ldots, m+1$, the $(i+1)$-th sub-network realize
  \begin{align*}
	z_{i} 
	&= z_{i-1}^s,
	& 
   z_{i,j} 
	&= z_{i-1}^j,\ j=1,\ldots, s-1; 
	&  
   y_{i,k} 
	&= \sum_{j=0}^{s-1} y_{i-1, ks+j} z_{i-1,j}, 
   \ \text{for}\ k=0, \ldots, n_{i}-1,
  \end{align*} 
  which, according to Lemma \ref{lem:baseeq} and Lemma \ref{lem:xny},
  can be realized by a neural network $\Phi^{i}_b$ of only one hidden
  layer.
  \begin{align}
	\label{eq:phi_bi}
	\Phi^i_{b} 
	= \Phi^1_{mo} \nnpar ( \Phi_{c}) \nnpar \Phi^{2,i}.
  \end{align}
  $\Phi^{2,i}$ is a network to realize
  $\left\{ y_{i,k} \mid k=0,\ldots, n_i-1 \right\}$. It is composed of
  $n_i$ sub-nets $\Phi^1_{bm,1}$ calculating multiplications.  The number
  of nodes, and numbers of nonzero weights in $\Phi^i_b$ are
  \begin{align}
	N(\Phi^i_b)
	&= 2+2s+4(s-1)n_i  
	&\\
	M_1(\Phi^i_b) 
	&= 2+4s+12(s-1) n_i 
	& 
   M_2(\Phi^i_b) 
	&= 2 + 2(s-1)s + 4s(s-1)n_i.
  \end{align}
  At the end of the iteration, we have $p(x) = y_{m+1,0}$.
	
\item[iv)] The overall network is obtained by taking a concatenation of all
  the sub-networks in each iteration.
  \begin{equation}
	\label{eq:phi_po3}
	\Phi^3_{po} 
    = \Phi^{m+1}_b \circ \Phi^m_b \circ \cdots \circ \Phi^0_b.
  \end{equation}
  This network has $m+3$ layers. A straightforward calculation gives us
  \begin{align}
	N(\Phi^3_{po}) 
	&= \sum_{k=1}^{m+1} N(\Phi^i_b) = (2+2s) + (2+4s) + \sum_{i=2}^{m+1} (2+2s+ 4s(s-1)n_i) 
   = \mathcal{O}\big( 4 n \big), \\
	M(\Phi^3_{po}) 
	&= (2+4s) + ( (4+6s)+4s^2+2s) 
   + ( (4+6s) + 4(s-1)n_1 (2s+2s+1))\nonumber\\
	&\quad{}+ \sum_{i=2}^{m} \Big( (4+6s) + 4(s-1)n_i (2s+4(s-1)+1) \Big)
   + N(\Phi^{m+1}_b) 
   = \mathcal{O}\big( (8s+16) n \big).
  \end{align}
\end{itemize}
The proof is complete. 
The overall construction is summarized in Algorithm \ref{alg:pnet}.
\end{proof}

\begin{algorithm}[htp!]
  \caption{PNet\_Polynomial$(n, s, \bm{a})$} %
  \label{alg:pnet}
  \hspace*{0.02in} {\bf Input:}   $n, \ s, \ \bm{a}=(a_{0}, \ a_1, ..., a_{n})$. \\
  \hspace*{0.02in} {\bf Output:}   $\Phi_{po}$ \qquad (with property $R_{\sigma_s}(\Phi_{po})(x) = \sum_{k=0}^n a_k x^n$)
  \begin{algorithmic}[1]
  	\If{$n \le s$} %
    \State Form $\Phi_{po} = \Phi^1_{po} $ given by \eqref{eq:poly1_net}. 
    \Else
    \State Let $m=\lint{\log_s n}$
    \State Form $\Phi^0_b$ given by \eqref{eq:phi_b0}
    \State Form $\Phi^1_b$ given by \eqref{eq:phi_b1}
    \For{$i=2$ to $m+1$} 
    \State 	Form $\Phi^i_b$ given by \eqref{eq:phi_bi}
    \EndFor
    \State Form $\Phi_{po} = \Phi_{po}^3$  given by (\ref{eq:phi_po3})
    \EndIf
    \Return $\Phi_{po}$.
  \end{algorithmic}
\end{algorithm}

\subsection{Error bounds of approximating univariate smooth functions}

Now we analyze the error of approximating general smooth functions using
RePU networks. Let $\Omega \subseteq \bbR^d$ be the domain on which the
function to be approximated is defined. For the one dimensional case, we focus
on $\Omega =I := [-1, 1]$. We denote the set of polynomials with degree up
to $N$ defined on $\Omega$ by ${P}_N(\Omega)$, or simply ${P}_N$. Let
$J^{\alpha, \beta}_{n}(x)$ be the Jacobi polynomial of degree $n$ for
$n=0, 1, \ldots$, which form a complete set of orthogonal bases in the
weighted $L^2_{\omega^{\alpha, \beta}}(I)$ space with respect to weight
$\omega^{\alpha, \beta}=(1-x)^{\alpha}(1+x)^{\beta}$, $\alpha,
\beta>-1$. To describe functions with high order regularity, we define
Jacobi-weighted Sobolev space $B_{\alpha, \beta}^{m} (I)$ as
\cite{shen_spectral_2011}:
\begin{equation}
  \label{eq:wtSobolevSpace}
  B_{\alpha, \beta}^{m} (I)
  := \left\{ u : \partial_{x}^{k}u 
    \in L_{\omega^{\alpha+k, \beta+k} }^2 (I), 
    \quad 0 \leq k \leq m 
  \right\}, \quad m \in \bbN_0, 
\end{equation}
with norm
\begin{equation}
  \label{eq:wtSobolevSpaceNorm}
  \| f \|_{B^{m}_{\alpha, \beta}} 
  := \left(
    \sum_{k=0}^m 
    \big\| \partial_x^k u \big\|^p_{L^2_{\omega^{\alpha+k, \beta+k}}}
  \right)^{1/2}.
\end{equation}
Define the $L^2_{\omega^{\alpha, \beta}}$-orthogonal projection
$\pi^{\alpha, \beta}_N$:
$L^2_{\omega^{\alpha, \beta}}(I) \rightarrow P_N$ as
\begin{equation}
  \left( \pi_N^{\alpha, \beta} u - u, v\right)_{\omega^{\alpha, \beta}} 
  = 0, 
  \quad
  \forall\, v\in P_N.
\end{equation}
A detailed error estimate on the projection error
$\pi_N^{\alpha, \beta}u - u$ is given in Theorem 3.35 of
\cite{shen_spectral_2011}, by which we have the following theorem on the
approximating error of general smooth functions using RePU networks.

\begin{theorem}\label{thm:s2approxN}
  Let $\alpha, \beta>-1$. For any $u\in B^{m}_{\alpha, \beta}(I)$, there
  exist a $\sigma_s$ network $\Phi^{u}_{N}$ with
  $L(\Phi^{u}_{N}) = \uint{\log_{s}N}+2$,
  $N(\Phi^{u}_{N}) =\mathcal{O}(N)$, $M(\Phi^{u}_{N})=\mathcal{O}(sN)$,
  satisfying the following estimate
  \begin{itemize}
  \item If~~$0 \le l\le m \le N+1$, we have
    \begin{align}
      \label{eq:err1d_proj1}
      \left\| \partial^{l}_{x}\left( R_{\sigma_{s}}(\Phi^{u}_{N}) - u
      \right) \right\|_{\omega^{\alpha+l, \beta+l}}
      &\leq
        c  \sqrt{ \dfrac{(N-m+1)!}{(N-l+1)!} }
        (N+m)^{(l-m)/2}  
        \|\partial^{m}_{x}u\|_{\omega^{\alpha+m, \beta+m}}, 
    \end{align}
  \item If~~$m>N+1$, we have
    \begin{align}
      \label{eq:err1d_proj2}
      \left\| \partial^{l}_{x}\left(
      R_{\sigma_{s}}(\Phi^{u}_{N}) - u
      \right)\right\|_{\omega^{\alpha+l, \beta+l}}
      &\le c (2\pi N)^{-1/4} \left(\dfrac{\sqrt{e/2}}{N}\right)^{N-l+1}
        \|\partial^{N+1}_{x}u\|_{\omega^{\alpha+N+1, \beta+N+1} }, 
    \end{align}
  \end{itemize}
  where $c\approx 1$ for $N\gg 1$.
\end{theorem}

\begin{proof}
  For any given $u\in B^{m}_{\alpha, \beta}(I)$, there exists a
  polynomials $f = \pi^{\alpha, \beta}_N u \in P_N$. The projection
  error $ \pi^{\alpha, \beta}_N u-u $ is estimated by Theorem 3.35 in
  \cite{shen_spectral_2011}, which is exactly \eqref{eq:err1d_proj1} and
  \eqref{eq:err1d_proj2} with $R_{\sigma_{s}}(\Phi^{u}_{N})$ replaced by
  $ \pi^{\alpha, \beta}_N u$.  By Theorem \ref{thm:Pn1d}, $f$ can be
  represented by a ReQU network (denoted by $\Phi^{u}_{N}$) with no error, i.e.
  $R_{\sigma_{s}}(\Phi^{u}_{N}) \equiv \pi^{\alpha, \beta}_N u$. We thus
  obtain estimate \eqref{eq:err1d_proj1} and \eqref{eq:err1d_proj2}.
\end{proof}

\begin{remark}
  Note that when $N \gg m$, the $L^2$ convergence rate given by
  \eqref{eq:err1d_proj1} is of order $\mathcal{O}\big(N^{-m}\big)$, which
  by the optimal nonlinear approximation theory developed by DeVore, Howard
  and Micchelli \cite{devore_optimal_1989}, is optimal if the network
  parameters depend continuously on the approximated function.
\end{remark}

Based on Theorem \ref{thm:s2approxN}, we can analyze the network
complexity of $\veps$-approximation of a given function with certain
smoothness. For simplicity, we only consider the case with
$\alpha=\beta=0, l=0$.  The result is given in the following theorem.

\begin{theorem}
  For any given function $f(x)\in B^{m}_{\alpha, \beta}(I)$ with norm
  less than $1$, where $m$ is either a fixed positive integer or
  infinity, there exists a RePU network $\Phi^f_\veps$ can approximate
  $f$ within an error tolerance $\veps$, i.e.
  \begin{align}
    \|R_{\sigma_{s}}(\Phi^{f}_{\veps}) - f\|_{L^{2}(I) }
    & \le
      \veps.
  \end{align}
  The number of layers $L$, numbers of nodes $N$ and nonzero weights $M$
  can be bounded as
  \begin{itemize}
  \item if $m$ is a fixed positive integer, then
    $L = \mathcal{O}\left(\frac{1}{m}\log_s\frac{1}{\veps}\right)$,
    $N = \mathcal{O}\big( {\veps}^{-\frac{1}{m}} \big)$ and
    $M = \mathcal{O}\big( s\,{\veps}^{-\frac{1}{m}} \big)$;
  \item if $m=\infty$, then
    $L= \mathcal{O}\left(\log_s\left(\ln\frac{1}{\veps}\right)\right)$,
    $N= \mathcal{O}\big(
    \frac{1}{\gamma_0}\ln\left(\frac{1}{\veps}\right) \big)$, and
    $M=\mathcal{O}\big(
    \frac{s}{\gamma_0}\ln\left(\frac{1}{\veps}\right) \big)$.  Here
    $\gamma_0 = \ln \ln \big( \frac{1}{\veps}\big)$.
  \end{itemize}
\end{theorem}

\begin{proof}
  For a fixed $m$, or $N\gg m$, we obtain from \eqref{eq:err1d_proj1}
  that
  \begin{equation}
    \| R_{\sigma_{s}}(\Phi^{u}_{N}) - u \|_{L^2} 
    \le c N^{-m} \| \partial_x^m u \|_{\omega^{\alpha+m, \beta+m}}.
  \end{equation}
  By above estimate, we obtain that to achieve an error tolerance
  $\veps$ to approximate a function with $B^{m}_{\alpha, \beta}(I)$ norm
  less than $1$, one need to take
  $N=\left(\frac{c}{\veps}\right)^\frac{1}{m}$. For fixed $m$, we have
  $N=\mathcal{O}\big( {\veps}^{-\frac{1}{m}} \big)$, the depth of the
  corresponding RePU network is
  $L=\mathcal{O}\left(\frac{1}{m}\log_s\frac{1}{\veps}\right)$, and the
  number of nonzero weights is $M=\mathcal{O}(s\, \veps^{-\frac1m})$.
  
  For $m=\infty$, from equation \eqref{eq:err1d_proj2}, we have
  \begin{equation}
    \| R_{\sigma_{s}}(\Phi^{u}_{N}) - u \|_{L^2} 
    \le c (2\pi N)^{-\frac{1}{4}}
    \left(\dfrac{\sqrt{e/2}}{N}\right)^{N+1}
    \| u\|_{B_{\alpha, \beta}^{\infty}}
    \le c' e^{-\gamma N} \| u\|_{B_{\alpha, \beta}^{\infty}}, 
  \end{equation}
  where $c'$ is a general constant, and $\gamma = (\ln N-\frac{1}{2})$
  can be larger than any fixed positive number for sufficient large $N$.
  To approximate a function with $B^{\infty}_{\alpha, \beta}(I)$ norm
  less than $1$ with error $\veps= c' e^{-\gamma N}$, one needs to take
  $N=\frac{1}{\gamma}\ln\left(\frac{c'}{\veps}\right)
  <(\ln\big(\frac{c'}{\veps}\big))$ for $N>e^{1.5}$, from which we get
  $\gamma = \mathcal{O}(\ln N) = \mathcal{O}\Big(\ln \ln \big(
  \frac{c'}{\veps}\big) \Big)$, thus
  $N=\mathcal{O}\big( \frac{1}{\gamma_0}\ln\left(\frac{1}{\veps}\right)
  \big)$.  The depth of the corresponding RePU network is
  $L=\mathcal{O}\left(\log_s\left(\ln\frac{1}{\veps}\right)\right)$. The
  number of nonzero weights is
  $\mathcal{O}\big(s
  \frac{1}{\gamma_0}\ln\left(\frac{1}{\veps}\right)\big)$.
\end{proof}

\section{Approximation of multivariate smooth functions}

In this section, we discuss the approximation of multivariate smooth
functions by RePU networks.  Similar to the univariate case, we first study
the representation of polynomials then discuss the results for general
smooth functions.

\subsection{Approximating multivariate polynomials}

\begin{theorem}\label{thm:mdpoly}
  If $f(x)$ is a multivariate polynomial with {\em total} degree $n$ on
  $\bbR^d$, then there exists a $\sigma_s$ neural network
  $\Phi^{d}_{mpo}$ having $d\uint{\log_{s} n}+1$ hidden layers with no
  more than $\mathcal{O}\big(\binom{n+d}{d}\big)$ activation functions and
  $\mathcal{O}\big(s \binom{n+d}{d}\big)$ non-zero weights, can represent $f$ with no
  error.
\end{theorem}

\begin{proof}
  1) We first consider the 2-dimensional case. Suppose
  $f(x, y)=\sum\limits_{i+j=0}^n a_{ij}x^iy^j$, 
  and $n\ge s+1$ (The cases $n\le s$ are similar but easier, so we omit
  here).  To represent $f(x, y)$ exactly with a $\sigma_s$ neural
  network basing the results on 1-dimensional case given in Theorem
  \ref{thm:Pn1d}, we first rewrite $f(x, y)$ as
  \begin{align}
    f(x, y)
    &=
      \sum_{i=0}^n \bigg( \sum_{j=0}^{n-i} a_{i, j} y^j \bigg) x^i
      =: \sum_{i=0}^n a^y_i x^i, 
      \quad \text{where}\quad
      a^y_{i} =
      \sum\limits_{j=0}^{n-i} a_{i, j} y^j.
  \end{align}
  So, to realize $f(x, y)$, we first realize $a^y_i$, $i=0, \ldots, n-1$
  using $n$ small $\sigma_s$ networks $\Phi^{y}_i$, $i=0, \ldots, n-1$,
  i.e.  $R_{\sigma_s} (\Phi^{y}_i) (y) = a^y_i$ for given input $y$; then
  use a $\sigma_s$ network $\Phi^{x}_n$ to realize the 1-dimensional
  polynomials $f(x, y) = \sum_{i=0}^n a^y_i x^i$.  There are two places
  need some technique treatments, the details are given below.
  \begin{enumerate}
  \item[(1)] Since $\Phi^{x}_n$ takes $a^y_i, i=0, \ldots, n$ and $x$ as
    input, so these quantities must be presented at the same layer of the
    overall neural network, because we do not want connections over
    disjointed layers.  By Theorem \ref{thm:Pn1d}, the largest depth of
    networks $\Phi^{y}_i, i=0, \ldots, n-1$ is $\uint{\log_s n}+2$, so we
    can lift $x$ to layer $\uint{\log_s n}+2$ using a concatenation of multiple $id_X(\cdot)$
    operations. Similarly, we also keep a record of input $y$ in each layer
    using multiple $id_X(\cdot)$, such that $\Phi^{y}_i, i=1, \ldots, n-1$
    can start from appropriate layer and generate output exactly at layer
    $\lfloor\log_s n\rfloor+2$. The overall cost for recording $x, y$ in
    layers $1, \ldots, \uint{\log_s n}+2$ is about $4s(\uint{\log_s n}+2)$,
    which is negligible comparing to the overall cost.
    
  \item[(2)] While realizing $\sum_{i=0}^n a^y_i x^i$, the coefficients
    $a^y_i, i=0, \ldots n$ are network input instead of fixed parameters.
    So when applying the network construction given in Theorem
    \ref{thm:Pn1d}, we need to modify the structure of the first and second
    layer of the network. i.e. using approach for $y_{i,k}, i\ge 2$ in
    \eqref{eq:phi_bi} for $y_{1,k}$ as well. This will increase the nodes
    in this layer from $\mathcal{O}(n)$ to $\mathcal{O}(sn)$, but since
    $n>s$, this does not change the overall scaling of the total number of
    nodes.
  \end{enumerate}
  
  By a direct calculation, we find the number of layers, number of nodes
  and nonzero weights in this realization can be bounded by
  $2\uint{\log_s n} + 2$, $\mathcal{O}\big( \binom{n+2}{2}\big)$, and
  $\mathcal{O}\big(s \binom{n+2}{2}\big)$.
  
  2) The case $d>2$ can be proved by mathematical induction using the
  similar procedure as done for $d=2$ case.
\end{proof}
Using similar approach as in Theorem \ref{thm:mdpoly}, one can easily
prove the following theorem.
\begin{theorem}\label{thm:mdpoly_tensor}
  For a polynomials $f_N$ in a tensor product space
  $Q_N^d(I_1\times\cdots \times I_d) := P_N(I_1)\otimes\cdots\otimes
  P_N(I_d)$, there exists a $\sigma_s$ network having
  $d\uint{\log_{s} N}+1$ hidden layers with no more than
  $\mathcal{O}(N^d)$ activation functions and $\mathcal{O}(s N^d)$ non-zero weights, can
  represent $f_N$ with no error.
\end{theorem}


\subsection{Error bound of approximations of multivariate smooth functions}

For a vector $\bx=(x_1, \ldots, x_d) \in \bbR^d$, we define
$|\bx|_1:=|x_1| +\ldots + |x_d|$, $|\bx|_{\infty} := \max_{i=1}^d |x_i|$.
Define high dimensional Jacobi weight
$\omega^{\bm{\alpha}, \bm{\beta}} := \omega^{\alpha_1,
  \beta_1}\cdots\omega^{\alpha_d, \beta_d}$.  We define multidimensional
Jacobi-weighted Sobolev space $B_{\alpha, \beta}^{m} (I^d)$ as
\cite{shen_spectral_2011}:
\begin{equation}
  \label{eq:wtSobolevSpaceMd}
  B_{\bm{\alpha}, \bm{\beta}}^{m} (I^d)
  := \left\{ u(\bx)
    \,\big|\,
    \partial^{\bm{k}}_{\bx} u
    := \partial_{x_1}^{k_1}\cdots\partial_{x_d}^{k_d} u
    \in L_{\omega^{\bm{\alpha}+\bm{k}, \bm{\beta}+\bm{k}} }^2 (I^d), 
    \quad  \bm{k}\in \bbN_0^d, \ |\bm{k}|_1 \le m \right\}, \quad m \in \bbN_0, 
\end{equation}
with norm and semi-norm
\begin{align}
  \label{eq:wtSobolevSpaceNormMd}
  \| u \|_{B^{m}_{\bm{\alpha}, \bm{\beta}}} 
  &:= \left(
	\sum_{0\le \ |\bm{k}|_1 \le m}
	\left\|
    \partial_{\bx}^{\bm{k}} u
	\right\|^2_{L^2_{\omega^{\bm{\alpha}+\bm{k}, \bm{\beta}+\bm{k}} }}
    \right)^{1/2}, 
  &
    | u |_{B^{m}_{\bm{\alpha}, \bm{\beta}}} 
  &:= \left(
	\sum_{\ |\bm{k}|_1 = m} \left\|
    \partial_{\bx}^{\bm{k}} u
	\right\|^2_{L^2_{\omega^{\bm{\alpha}+\bm{k}, \bm{\beta}+\bm{k}} }}
    \right)^{1/2}.
\end{align}

Define the $L^2_{\omega^{\bm{\alpha}, \bm{\beta}}}$-orthogonal
projection $\pi^{\bm{\alpha}, \bm{\beta}}_N$:
$L^2_{\omega^{\bm{\alpha}, \bm{\beta}}}(I^d) \rightarrow Q_N^d(I^d)$ as
\begin{equation*}
  \left( \pi_N^{\bm{\alpha}, \bm{\beta}} u - u, 
	v\right)_{\omega^{\bm{\alpha}, \bm{\beta}}} 
  = 0, 
  \quad
  \forall\, v\in P_N^d(I^d).
\end{equation*}
Then for $u\in {B^{m}_{\bm{\alpha}, \bm{\beta}}}$, we have the following
error estimate(see e.g. \cite{shen_spectral_2011})
\begin{equation}
  \label{eq:Mdpolyapprox}
  \| \pi_N^{\bm{\alpha}, \bm{\beta}} u - u
  \|_{L^2_{\omega^{\bm{\alpha}, \bm{\beta}}}(I^d)}
  \le c N^{-m}
  | u|_{{B^{m}_{\bm{\alpha}, \bm{\beta}}}}, 
  \quad 1\le m \le N, 
\end{equation}
where $c$ is a general constant. Combining \eqref{eq:Mdpolyapprox} and
Theorem \ref{thm:mdpoly_tensor}, we reach to the following upper bound
for the $\veps$-approximation of functions in
${B^{m}_{\bm{\alpha}, \bm{\beta}}}(I^d)$ space.

\begin{theorem}\label{thm:MdSobolev}
  For any $u\in {B^{m}_{\bm{\alpha}, \bm{\beta}}}(I^d)$, with
  $|u|_{{B^{m}_{\bm{\alpha}, \bm{\beta}}}(I^d)} \le 1$, there exists a
  $\sigma_s$ neural network $\Phi_\veps^u$ having
  $\mathcal{O}\left(\frac{d}{m} \log_s \frac{1}{\veps} + d\right)$
  hidden layers with no more than $\mathcal{O}\left(\veps^{-d/m}\right)$
  nodes and $\mathcal{O}\left(s\, \veps^{-d/m}\right)$ non-zero weights,
  approximate $u$ with ${L^2_{\omega^{\bm{\alpha}, \bm{\beta}}}(I^d)}$
  error less than $\veps$, i.e.
  \begin{equation}
    \| R_{\sigma_s} (\Phi_\veps^u) - u
    \|_{L^2_{\omega^{\bm{\alpha}, \bm{\beta}}}(I^d)}
    \le \veps.
  \end{equation}
\end{theorem}

\subsection{High-dimensional smooth functions with sparse polynomial
  approximations}

In last section, we showed that for a $d$-dimensional functions with
partial derivatives up to order $m$ in $L^2(I^d)$ can be approximated
within error $\veps$ by a RePU neural network with complexity
$\mathcal{O}(\veps^{-d/m})$. When $m$ is much smaller than $d$, we see the
network complexity has an exponential dependence on $d$. However, in a lot
of applications, high-dimensional problem may have low intrinsic dimension
\cite{wang_why_2005}, for those applications, we may first do a dimension
reduction, then use the $\sigma_s$ neural network construction proposed
above to approximate the reduced problem.  On the other hand,
for high-dimensional functions with bounded mixed derivatives, we can use
{\em sparse grid} or {\em hyperbolic cross} approximation to lessen the
curse of dimensionality.

\subsubsection{A brief review on hyperbolic cross approximations}

We introduce hyperbolic cross approximation by considering a tensor product
function: $f(\bx) = f_1(x_1)f_1(x_2)\cdots f_d(x_d)$. Suppose that
$f_1, \ldots, f_d$ have similar regularity that can be well approximated by
using a set of orthonormal bases $\{ \phi_k, k=1, 2, \ldots. \}$ as
\begin{equation}
  f_i(x) 
  = \sum_{k=0}^\infty b_k^{(i)} \phi_k(x), 
  \quad |b_k^{(i)}|\le c \bar{k}^{-r},  \quad i=1, 2, \ldots, d, 
\end{equation}
where $c$ and $r\ge 1$ are constants depending on the regularity of
$f_i$, $\bar{k}:=\max\{1, k\}$. So we have an expansion for $f$ as
\begin{equation}
  f(\bx)
  = \prod_{i=1}^d \left( \sum_{k=0}^\infty b_k^{(i)} \phi_k(x_i) \right)
  = \sum_{\bm{k}\in \bbN_0^d} b_{\bm{k}} \phi_{\bm{k}}(\bx), 
  \quad \text{where}\
  |b_{\bm{k}}|
  = \big|b_{k_1}^{(1)}\cdots b_{k_d}^{(d)} \big|
  \le c^d (\bar{k}_1\cdots \bar{k}_d)^{-r}.
\end{equation}
Thus, to have a best approximation of $f(\bx)$ using finite terms, one
should take
\begin{equation}
  \label{eq:HyperbolicCrossDef} f_N := \sum_{\bm{k}\in
    {\chi}_N^d} b_{\bm{k}} \phi_{\bm{k}}(\bx), 
\end{equation}
where
\begin{equation}
  {\chi}_N^d 
  := \left\{
	\bm{k}=(k_1, \ldots, k_d)  \in \bbN_0^d
	\mid
	\bar{k}_1\cdots \bar{k}_d \le N
  \right\}
\end{equation}
is the hyperbolic cross index set. We call $f_N$ defined by
\eqref{eq:HyperbolicCrossDef} a hyperbolic cross approximation of $f$.

For general functions defined on $I^d$, we choose $\phi_{\bm{k}}$ to be
multivariate Jacobi polynomials $J_{\bm{n}}^{\bm{\alpha}, \bm{\beta}}$,
and define the hyperbolic cross polynomial space as
\begin{equation} 
  \label{eq:hyperbolic_poly}
  X^d_N 
  := \text{span}\big\{\, J_{\bm{n}}^{\bm{\alpha},\bm{\beta}}, 
  \:
  \bm{n}\in \chi^d_N
  \, \big\}.
\end{equation}
Note that the definition of $X_N^d$ doesn't depends on $\bm{\alpha}$ and
$\bm{\beta}$.  $\{ J_{\bm{n}}^{\bm{\alpha}, \bm{\beta}}\, \}$ is used to
served as a set of bases for $X_N^d$.  To study the error of hyperbolic
cross approximation, we define Jacobi-weighted Korobov-type space
\begin{equation}
  \label{eq:Korobov}
  \mathcal{K}^m_{\bm{\alpha}, \bm{\beta}}(I^d)
  := \left\{\, 
	u(\bx)\ : \ \partial^{\bm{k}}_{\bx} u  \in
	L^2_{\omega^{\bm{\alpha}+\bm{k}, \bm{\beta}+\bm{k}}}(I^d), 
	\
	0\le \ |\bm{k}|_\infty \le m
	\, \right\}, 
  \quad \text{for}\ m\in\bbN_0, 
\end{equation}
with norm and semi-norm
\begin{align}
  \label{eq:KorobovNormMd}
  \| u \|_{\mathcal{K}^{m}_{\bm{\alpha}, \bm{\beta}}} 
  &:= \left(
	\sum_{0\le \ |\bm{k}|_\infty \le m}
	\left\|
    \partial_{\bx}^{\bm{k}} u
	\right\|^2_{L^2_{\omega^{\bm{\alpha}+\bm{k}, \bm{\beta}+\bm{k}} }}
    \right)^{1/2}, 
  &
    | u |_{\mathcal{K}^{m}_{\bm{\alpha}, \bm{\beta}}} 
  &:= \left(
	\sum_{\ |\bm{k}|_\infty = m} \left\|
    \partial_{\bx}^{\bm{k}} u
	\right\|^2_{L^2_{\omega^{\bm{\alpha}+\bm{k}, \bm{\beta}+\bm{k}} }}
    \right)^{1/2}.
\end{align}
For any give
$u\in \mathcal{K}^0_{\bm{\alpha}, \bm{\beta}} (=B^0_{\bm{\alpha},
  \bm{\beta}})$, the hyperbolic cross approximation can be defined as a
projection as
\begin{equation}
  \label{eq:hyperbolic_proj}
  (\pi_{N, H}^{\bm{\alpha}, \bm{\beta}} u -u, v)_{\omega^{\bm{\alpha}, 
      \bm{\beta}}} 
  = 0, 
  \quad
  \forall\, v\in X_N^d.
\end{equation}
Then we have the following error estimate about the hyperbolic cross
approximation \cite{shen_sparse_2010}:
\begin{equation}
  \label{eq:hyperbolic_error}
  \| \partial^{\bm{l}}_{\bx}
  (\pi_{N, H}^{\bm{\alpha}, \bm{\beta}} u - u)
  \|_{\omega^{\bm{\alpha+l}, \bm{\beta+l}}}
  \le D_1 N^{|\bm{l}|_\infty - m}
  |u|_{\mathcal{K}^{m}_{\bm{\alpha}, \bm{\beta}}}, 
  \quad 0\le \bm{l} \le \bm{m}, \ \bm{m}\ge 1, 
\end{equation}
where $D_1$ is a constant independent of $N$.  It is known that the
cardinality of $\chi_N^d$ is of order $\mathcal{O}(N(\log
N)^{d-1})$. The above error estimate says that to approximation a
function $u$ with
$|u|_{\mathcal{K}^{m}_{\bm{\alpha}, \bm{\beta}}} \le 1/D_1$ with an
error tolerance $\veps$, one need no more than
$\mathcal{O}\left(\veps^{-1/m}(\frac{1}{m} \log
  \frac{1}{\veps})^{d-1}\right) $ Jacobi polynomials, the exponential
dependence on $d$ is weakened.

In practice, the exact hyperbolic cross projection is not easy to
calculate. An alternate approach is the sparse
grids\cite{smolyak_quadrature_1963,bungartz_sparse_2004}, which use
hierarchical interpolation schemes to build an hyperbolic cross like
approximation of high dimensional functions
\cite{barthelmann_high_2000,shen_efficient_2010}.

\subsubsection{Error bounds of approximating some high-dimensional smooth
  functions}

Now we discussion the RePU network approximation of high-dimensional smooth
functions. Our approach bases on high-dimensional hyperbolic cross polynomial
approximations.  We introduce a concept of { \em complete} polynomial space
first. A linear polynomial space $P_{C}$ is said to be complete if it
satisfies the following: There exists a set of bases composed of only
monomials belonging to $P_{C}$, and for any term $p(x)$ in this basis set,
all of its derivatives $ \partial^{\bm{k}}_{\bx} p(\bx)$,
$\bm{k}\in \bbN_0^d$ belongs to $P_{C}$. It is easy to verify that both the
hyperbolic cross polynomial space $X^d_N$ and sparse grid polynomial
interpolation space $V^q_d$ (see
\cite{shen_efficient_2010,shen_efficient_2012}) are complete.  For a
complete polynomial space, we have the following RePU network
representation results.

\begin{theorem}
  \label{thm:PC_RePU}
  Let $P_{C}$ be a complete linear space of $d$-dimensional polynomials
  with dimension $n$, then for any function $f\in P_{C}$, there exists a
  $\sigma_s$ neural network having no more than
  $\sum_{i=1}^d\uint{\log_{s} N_i}+1$ hidden layers, no more than
  $\mathcal{O}(n)$ activation functions and $\mathcal{O}(s n)$ non-zero
  weights, can represent $f$ with no error. Here $N_i$ is the maximum
  polynomial degree in $i$-th dimension in $P_{C}$.
\end{theorem}

\begin{proof}
  The proof is similar to Theorem \ref{thm:mdpoly}. First, $f$ can be
  written as linear combinations of monomials.
  \begin{equation}
    f(\bx) 
    = \sum_{\bm{k}\in \chi_C} a_{\bm{k}} \bx_{\bm{k}}, 
  \end{equation}
  where $\chi_C$ is the index set of $P_C$ with cardinality $n$. Then
  we rearrange the summation as
  \begin{equation}
    f(\bx) 
    = \sum_{k_d=0}^{N_d}
    a_{k_d}^{x_1 \cdots x_{k_{d-1}}}
    x_d^{k_d}, 
    \quad
    a_{k_d}^{x_1\cdots x_{k_{d-1}}}
    :=
    \sum_{(k_1,\ldots, k_{d-1})\in \chi_C^{k_d}}
    a_{k_1 \cdots k_{d-1}} x_1^{k_1} \cdots x_{d-1}^{k_{d-1}}, 
  \end{equation}
  where $\chi_C^{k_d}$ are $d-1$ dimensional complete index sets that
  depend on the index $k_d$.  If each term in
  $a_{k_d}^{x_1\cdots x_{k_{d-1}}}$, $k_d=0, 1, \ldots, N_d$ can be
  exactly represented by a $\sigma_s$ network with no more than
  $\sum_{i=1}^{d-1}\uint{\log_{s} N_i}+1$ hidden layers, no more than
  $\mathcal{O}(\text{card}(\chi_C^{k_d}))$ nodes and
  $\mathcal{O}(s\cdot\text{card}(\chi_C^{k_d}))$ non-zero weights, then
  $f(x)$ can be exactly represented by a $\sigma_s$ neural network with
  no more than $\sum_{i=1}^d\uint{\log_{s} N_i}+1$ hidden layers, no
  more than $\mathcal{O}(n)$ nodes and non-zero weights.  So, by
  mathematical induction, we only need to prove that when $d=1$ the
  theorem is satisfied, which is true by Theorem \ref{thm:Pn1d}.
\end{proof}

\begin{remark} \label{rmk:spg_hyperbolic_ReQU} According to Theorem
  \ref{thm:PC_RePU}, we have that: For any $f\in X^d_N$, there is a RePU
  network having no more than $d\uint{\log_{s} N}+1$ hidden layers, no more
  than $\mathcal{O}(N(\log N)^{d-1})$ activation functions and
  $\mathcal{O}(s\,N(\log N)^{d-1})$ non-zero weights, can represent $f$
  with no error.  Combine the results with \eqref{eq:hyperbolic_error} and
  we can obtain the following theorem.
\end{remark}

\begin{theorem}
  \label{thm:MdsparseS2error}
  For any function $u\in \mathcal{K}^{m}_{\bm{\alpha}, \bm{\beta}}(I^d)$,
  $m\ge 1$ with
  $|u|_{\mathcal{K}^{m}_{\bm{\alpha}, \bm{\beta}}} \le 1/D_1$, any
  $\veps \ge 0$, there exists a RePU network $\Phi_\veps^u$ with no more
  than $ d \uint{\frac{1}{m} \log_s \frac{1}{\veps}} + 2$ layers, no more
  than
  $\mathcal{O}\big( \veps^{-1/m}(\frac{1}{m} \log_{s}
  \frac{1}{\veps})^{d-1} \big)$ nodes and
  $\mathcal{O}\big(s\, \veps^{-1/m}(\frac{1}{m}
  \log_{s}\frac{1}{\veps})^{d-1} \big)$ non-zero weights, such that
  \begin{equation}
    \| R_{\sigma_s} (\Phi_\veps^u) - u
    \|_{\omega^{\bm{\alpha}, \bm{\beta}}}
    \le \veps.
  \end{equation}
\end{theorem}

\begin{remark}
  Here, we bound the weighted $L^2$ approximation error by using the
  corresponding hyperbolic cross spectral projection error estimation
  developed in \cite{shen_sparse_2010}.  However, high-dimensional
  hyperbolic cross spectral projection is hard to calculate.  In practice,
  we use efficient sparse grid spectral transforms developed in
  \cite{shen_efficient_2010} and \cite{shen_efficient_2012} to approximate
  the projection. After a numerical network is built, one may further train
  it to obtain a network function that is more accurate than the sparse
  grid interpolation. Note that the fast sparse transform can be extended
  to tensor-product unbounded domain using the mapping method
  \cite{shen_approximations_2014}.
\end{remark}

\section{Summary}

In this paper, deep neural network realizations of univariate polynomials
and multivariate polynomials using general RePU as activation functions are
proposed with detailed constructive algorithms.  The constructed RePU
neural networks have optimal number of hidden layers and optimal number of
activation nodes.  By using this construction, we also prove some optimal
upper error bounds of approximating smooth functions in Sobolev space using
RePU networks. The optimality is indicated by the optimal nonlinear
approximation theory developed by DeVore, Howard and Micchelli for the case
that the network parameters depend continuously on the approximated function.  The constructive proofs reveal clearly the close connection
between the spectral method and deep RePU network approximation.

Even though we did not apply the proposed RePU networks to any real
applications in this paper, the good properties of the proposed networks
suggest that they have potential advantages over other types of networks in
approximating functions with good smoothness.  In particular, it suits
situations where the loss function contains some derivatives of the network
function, in such a case, deep ReLU networks are known hard to use with
usual training methods.

\section*{Appendix}

The appendix section is devoted to proof Lemma \ref{lem:xny}. We first
present the following lemma which can be proved by induction.

\begin{lemma}\label{lem:identity}
  For $s\in \bbN$ we have
  \begin{align}
  	\label{eq:ident}
    (2^{s-1}s!) \prod^{s}\limits_{k=1}x_{k}
    &= 
      \left(\sum^{s}_{k=1} x_k \right)^s 
      +
      \sum^{s-1}_{k=1}(-1)^k \!\!\!\sum_{1<i_1 < \cdots <i_k} S_{i_1, \ldots, i_k},
  \end{align}
  where
  \begin{align*}
    S_{i_1, \ldots, i_k}
    &= \big(x_1 +\cdots +(-1)x_{i_1} 
      +\cdots +(-1)x_{i_k}+\cdots +x_s \big)^s.
  \end{align*}
\end{lemma}

\begin{corollary}\label{cor:n1_n2}
  For $s\in \bbN$ and
  $n_{1}+n_{2}=t\in \{\, 0, 1, \ldots, s \,\}, \ n_1, n_{2} \in\{\, 0,
  1, \ldots, t \,\}$, we have
  \begin{align}
    (2^{s-1}s!)x^{n_1}y^{n_2}
    &= 
      \big[n_{1}x + n_{2}y + (s-t) \big]^s +
      \sum^{s-t}_{k=1}(-1)^{k}\binom{s-t}{k}\left[ n_{1}x + n_{2}y + (s-t-2k) \right]^s \nonumber\\
    &\quad{}
      +
      \sum^{s-1}_{k=1}(-1)^k \!\!\!
      \sum^{\min\{t-1, k\}}_{r=\max\{1, k-(s-t)\}}
      \sum^{\min\{r, n_1-1 \} }_{j=\max\{0, r-n_2 \}}
      \binom{s-t}{k-r}\binom{n_{1}-1}{j}\binom{n_2}{r-j}
      S_{j, r, k}^{s, n_1, n_2},
      \label{eq:xn1yn2}
  \end{align}
  where
  \begin{align}
    S_{j, r, k}^{s, n_1, n_2}
    &:= \Big[
      (n_1-2j)x + \big(n_2-2(r-j)\big)y + \big( s-(n_1+n_2)-2(k-r) \big)
      \Big]^s.
      \label{eq:Sjrk}
  \end{align}
\end{corollary}
\begin{proof}
  We take $x_1 = \cdots = x_{n_1} = x$, $x_{n_1+1} = \ldots = x_t = y$,
  $x_{t+1}=\cdots = x_s=1$ in Lemma \ref{lem:identity}. Denote
  \begin{align}
    A_{t} = \left\{ x_1, x_{2}, \ldots, x_{t} \right\}, 
    \quad
    B_{k} = \left\{ x_{i_1}, x_{i_2}, \ldots, x_{i_{k}} \right\},
  \end{align}
  and let $\#(A_{t}\cap B_{k})$ be the number of elements in both
  $A_{t}$ and $B_{k}$.  Then the second term on the left hand side of
  \eqref{eq:ident} can be summed in two groups:
  \begin{itemize}

  \item The first group include the cases that no term in $B_k$ is
    included in $A_t$, so we get $x_{i_1}=\cdots=x_{i_k}=1$. Each
    $S_{i_1, \ldots, l_k}$ term in this case is equal to
    $\left[ n_{1}x + n_{2}y + (s-t-2k) \right]^s$, there are
    $\binom{s-t}{k}$ such terms.

  \item The second group includes the cases that there exist at least
    one term in $B_k$ is contained in $A_t$.  We let
    $r=\#(A_{t}\cap B_{k})>0$, $j=\#(A_{n_1}\cap B_{k})$, then we have
    \begin{align*}
      \max\{ 1, k-(s-t) \} \le & r \le \min\{ t-1, k\}, \\
      \max\{ 0, r-n_2 \} \le & j \le \min\{ r, n_1 - 1\}.
    \end{align*}
    Each $S_{i_1, \ldots, l_k}$ term in this case is equivalent to
    \eqref{eq:Sjrk}. There are in total
    $\binom{s-t}{k-r} \binom{n_1-1}{j} \binom{n_2}{r-j} $ such terms.
  \end{itemize}

  Summing up all the terms, we obtain the identity \eqref{eq:xn1yn2}
\end{proof}

\begin{proof} [Proof of Lemma \ref{lem:xny}]
  First, by taking $n_{1}=1$ ,$n_{2}= n =t-1$ in Corollary \ref{cor:n1_n2}
  and exchange the positions of $x,y$, we get
  \begin{align*}
    (2^{s-1}s!) y x^n
    &=
      \Big[ nx + y + (s-(n+1)) \Big]^{s} 
      + 
      \sum^{s-(n+1)}_{k=1}(-1)^{k}\binom{s-(n+1)}{k}\Big[n x + y + (s-(n+1)-2k) \Big]^{s} \\
    & \quad{}
      + \sum^{s-1}_{k=1}(-1)^k \!\!\!
      \sum^{\min\{n, k\}}_{r=\max\{1, k-(s-t)\}} \!\!\!
      \binom{s-t}{k-r}\binom{n}{r}
      \Big[ \big(n-2r\big)x + y + \big( s-t-2(k-r) \big) \Big]^s\\
    &= 
      \sum^{s-(n+1)}_{k=0}(-1)^{k}\binom{s-(n+1)}{k}\Big[n x + y + (s-(n+1)-2k) \Big]^{s} \\
    &\quad{} 
      + \sum^{s-(n+1)}_{j=0} \sum^{n}_{r=1} (-1)^{j+r}
      \binom{s-(n+1)}{j}\binom{n}{r}
      \Big[ \big(n-2r\big)x + y + \big( s-(n+1)-2j \big) \Big]^s \\
    &=
      \sum^{s-(n+1)}_{j=0} \sum^{n}_{r=0} (-1)^{j+r}
      \binom{s-(n+1)}{j}\binom{n}{r}
      \Big[ \big(n-2r\big)x + y + \big( s-(n+1)-2j \big) \Big]^s
  \end{align*}

  From above derivation, we see that $x^n y$ can be represented as a linear
  combination of $(n+1)\times (s-n)$ $\rho_s(\cdot)$ terms, that is
  \begin{align}
    \label{eq:xny_rhos}
    x^n y 
    &= \sum^{s-(n+1)}_{j=0} \sum^{n}_{r=0} \gamma^{j,r}_{s,n} 
      \rho_s \left( (n-2r)x + y + ( s-(n+1)-2j ) \right),
  \end{align}
  where
  \begin{equation}
    \label{eq:xny_gammas}
    \gamma^{j,r}_{s,n} 
    = \frac{(-1)^{j+r}}{(2^{s-1}s!)} \binom{s-(n+1)}{j}\binom{n}{r}.
  \end{equation}
  Denote by $z_{k+1} := (k, k-2, \ldots, -k)^T \in \bbR^{k+1}$, and
  $\bm{1}_k := (1,1, \ldots, 1)^T \in \bbR^{k}$, for $k\in \bbZ$.  For a
  matrix $A=\big(a_{kj}\big)_{k=1,m}^{j=1,n} \in \bbR^{m\times n}$,
  define its vectorization
  $\vect(A) := (a_{11}, \ldots, a_{m1}, \ldots, a_{1n}, \ldots,
  a_{mn})^T$. For two vectors $a\in \bbR^m$, $b\in \bbR^n$, define
  $a\otimes b := \big(a_i b_j\big)_{i=1,m}^{j=1,n} \in \bbR^{m\times
    n}$.  Denote
  $\Gamma_{s,n} = \big( \gamma_{s,n}^{j,r}\big)_{j=0,s-(n+1)}^{r=0,n}
  \in \bbR^{(s-n)\times (n+1)}$. Using these definitions and notations,
  \eqref{eq:xny_rhos} can be written as
  \begin{align}
    x^n y 
    &= \gamma^T_{2, n} 
      \sigma_s (\alpha_{2, n, 1} x + \alpha_{2, n, 2} y + \beta_{2, n}),
  \end{align}
  where
  \begin{align}
    \label{eq:xny_sigs_coef}
    \left\{
    \begin{aligned}
      \gamma_{2,n}
      &= \vect\left( \gamma_0\otimes \vect(\Gamma_{s,n})
        \right),
      & \alpha_{2,n,1}
      &= \vect\big( \alpha_0\otimes\vect(
        \bm{1}_{s-n}\otimes z_{n+1} ) \big),
      \\
      \alpha_{2,n,2}
      &= \vect\left( \alpha_0 \otimes
        \bm{1}_{(s-n)(n+1)}\right),
      & \beta_{2,n}
      &= \vect\big(
        \alpha_0\otimes\vect( z_{s-n}\otimes \bm{1}_{n+1} ) \big).
    \end{aligned}
        \right.
  \end{align}
  The length of those coefficients are all $2(s-n)(n+1)$.  The lemma is
  proved.
\end{proof}

\section*{Acknowledgments}

The last author is indebted to Prof. Jie Shen and Prof. Li-Lian Wang for
their stimulating conversations on spectral methods. The authors would like
to think Prof. Christoph Schwab and Prof. Hrushikesh N. Mhaskar for
providing us some related references. This work was partially supported by
China National Program on Key Basic Research Project 2015CB856003, NNSFC
Grant 11771439, 91852116, and China Science Challenge Project,
no. TZ2018001.

\bibliographystyle{unsrt}
\bibliography{dnn}

\end{document}